\documentclass[11pt]{article}
\usepackage[margin=1in]{geometry}
\usepackage[utf8]{inputenc}
\usepackage[T1]{fontenc}
\usepackage{amsmath,amssymb,amsfonts,amsthm}
\newtheorem{proposition}{Proposition}
\newtheorem{corollary}{Corollary}
\usepackage{graphicx}
\usepackage{booktabs}
\usepackage{xcolor}
\usepackage{colortbl}
\usepackage{hyperref}
\usepackage{lmodern}
\usepackage{microtype}
\usepackage{multirow}
\usepackage{natbib}
\setcitestyle{authoryear,round}

\hypersetup{colorlinks=true, linkcolor=blue, citecolor=blue, urlcolor=blue}

\newcommand{\dmodel}{d_{\text{model}}}

\title{Multiplicative Modulation for Adaptation}

\author{
  Hengshuai Yao\textsuperscript{1,2} \quad Xing Chen\textsuperscript{1} \quad Ahmed Murtadha\textsuperscript{1} \quad Guan Wang\textsuperscript{1} \\
  \textsuperscript{1}Sapient Intelligence \\
  \textsuperscript{2}Department of Computing Science, University of Alberta \\
  \texttt{hengshu1@ualberta.ca, raincchio@gmail.com,} \\
  \texttt{murtadha20@gmail.com, imonenext@gmail.com}
}

\begin{document}
\maketitle

\begin{abstract}
Adapting LLMs to new domains causes forgetting because standard methods (e.g., full fine-tuning, LoRA) inject new directions into the weight space. We show that forgetting is governed by one algebraic property: whether the update preserves the column span of the pretrained weight matrix (Proposition~\ref{prop:colspan}). We propose \textbf{GAIN}, the simplest multiplicative alternative ($W_{\text{new}} = S \cdot W$), which satisfies this by construction and can be absorbed into existing weights for zero inference cost. Across five models (774M to 70B) adapted sequentially over eight domains, GAIN \emph{improves} earlier-domain perplexity by $7$--$13\%$, while LoRA degrades it by $18$--$36\%$. GAIN matches replay-augmented LoRA without storing prior data and dominates EWC on the forgetting--adaptation Pareto front. While LoRA can only reduce forgetting by sacrificing in-domain adaptation, GAIN achieves both with no domain boundaries and no regularization. The principle generalises: (IA)$^3$, an independent multiplicative method, also improves earlier domains.
\end{abstract}

% ============================================================
\section{Introduction}

Large language models are pretrained on diverse web-scale corpora and then adapted to specialised domains. The standard view treats adaptation as a parameter-update problem: full fine-tuning updates all weights \citep{gururangan2020dontstop}; LoRA \citep{hu2022lora} adds a low-rank correction $W \!\leftarrow\! W + BA$. Both cause \emph{forgetting}---earlier-trained capabilities degrade as new ones are learned~\citep{biderman2024lora}---and the field treats this as an inevitable cost of capacity.

We show that, surprisingly, forgetting is due to a design choice. A wrapped projection $W \!\leftarrow\! f(W)$ preserves earlier-trained capabilities if $\mathrm{col}(f(W)) \subseteq \mathrm{col}(W)$ (Proposition~\ref{prop:colspan}). Diagonal scaling $W \!\leftarrow\! S \!\cdot\! W$ satisfies this by construction; additive correction $W \!\leftarrow\! W + BA$ does not, except in the degenerate $\mathrm{rank}(B)\!=\!0$ case. The difference for performance is sharp: across the experiments reported in this paper, every additive method we test degrades earlier-trained domains, while all three multiplicative methods improve them.

We propose \textbf{GAIN}---the simplest multiplicative instantiation---as a one-line replacement for additive PEFT that can be absorbed into existing weights for zero inference cost. GAIN learns a diagonal gain vector $S$ per attention output projection ($46$K parameters on GPT-2 Large, $131$K on Llama-3-8B); GAIN-FFN extends to the FFN output ($230$K, $590$K). The principle is inspired by \emph{gain modulation} in cortex \citep{salinas2000gain, andersen1985gain}: neurons adapt by scaling response strength without changing what they respond to, preserving the population code over time.

Our four contributions:
\begin{enumerate}
    \item \textbf{The colspan dichotomy.} Forgetting is determined by whether the adapted weight matrix stays within the pretrained output space. Multiplicative scaling preserves it; additive correction does not (Proposition~\ref{prop:colspan}). This formalises the ``intruder dimensions'' observation of \citet{biderman2024lora}; we verify it by direct subspace measurement (Appendix~\ref{sec:mechanistic}, Table~\ref{tab:mechanistic}): after sequential adaptation, LoRA's adapted weight retains only $0.53$ top-10 subspace overlap with the pretrained $W$, while GAIN-FFN preserves $0.91$.
    \item \textbf{Multiplicative preserves prior domains far better than additive.} On the eight-domain sequential protocol, GAIN-FFN improves earlier-trained domains by $7$--$21\%$ while LoRA degrades them by $1$--$57\%$ (three-seed standard deviation $\leq 0.5$ percentage points). Extending to nineteen domains, the gap persists: GAIN-FFN improves by $14.7\%$, the strongest LoRA variant by only $2.0\%$.
    \item \textbf{GAIN alone matches CL baselines augmented with replay, domain boundaries, or regularization.} GAIN-FFN matches Replay-LoRA's preservation without storing prior data, dominates EWC on the forgetting--adaptation Pareto front, and beats PackNet-LoRA by $17$ percentage points at matched parameter budget---using $3\!\times\!$ fewer trainable parameters than LoRA $r\!=\!8$ and requiring no domain-boundary knowledge.
    \item \textbf{The multiplicative principle, not GAIN's parameterisation, prevents forgetting.} Multiplicative methods trade a small ceiling on in-domain adaptation for stable preservation of prior capabilities. (IA)$^3$ \citep{liu2022ia3}, an independent multiplicative method that scales attention K/V and FFN gates, confirms this: it also improves earlier domains ($-14.2\%$).
\end{enumerate}

GAIN generalises to instruction-tuned models (Llama-3-8B-Instruct, Appendix~\ref{sec:llama-instruct}) and remains effective even on out-of-distribution domains (English-only GPT-2 Large adapted to German Wikipedia, Appendix~\ref{sec:ood-test}).

% ============================================================
\section{Methods}
\label{sec:method}

We propose two methods: GAIN scales the attention output projection, and GAIN-FFN extends this to the FFN. Both instantiate the multiplicative principle:
\begin{equation}
    W_{\text{new}} = S \cdot W
    \label{eq:principle}
\end{equation}
where $S$ is a diagonal matrix that scales rows of $W$.

\subsection{GAIN}

We instantiate Eq.~\ref{eq:principle} for the attention output projection $W_O$. A transformer with $L$ layers and $H$ heads computes:
\begin{equation}
    \text{output} = [\text{head}_1, \ldots, \text{head}_H] \, W_O
\end{equation}
where $W_O \in \mathbb{R}^{\dmodel \times \dmodel}$ projects concatenated head outputs into the residual stream. GAIN inserts a diagonal scaling matrix before $W_O$:
\begin{equation}
    \text{output} = [\text{head}_1, \ldots, \text{head}_H] \, S \, W_O, \quad S = \text{diag}(s_1, \ldots, s_{\dmodel})
\end{equation}
This gives $\dmodel$ parameters per layer, $\dmodel \times L$ total. For GPT-2 Large ($\dmodel\!=\!1280$, $L\!=\!36$), this is 46,080 parameters.

We freeze all pretrained weights and initialize $S = I$, so that training starts from the pretrained model. After training, $S$ is absorbed into the weights: $W_O \leftarrow S \cdot W_O$, adding zero inference cost.

\subsection{GAIN-FFN}

Eq.~\ref{eq:principle} scales rows of $W$ (left multiplication). The complementary form scales columns:
\begin{equation}
    W^{\text{new}} = W \cdot S
    \label{eq:right_mult}
\end{equation}
GAIN-FFN is an instance of Eq.~\ref{eq:right_mult} applied to the FFN. Each transformer layer contains a feed-forward network: the up-projection $W_{\text{up}}$ maps from $\dmodel$ to $d_{\text{ffn}}$ (typically $4\times$ larger), a nonlinearity is applied, and the down-projection $W_{\text{down}}$ maps back. GAIN-FFN scales the $d_{\text{ffn}}$-dimensional intermediate activations:
\begin{equation}
    \text{FFN}(h) = W_{\text{down}} \cdot S_{\text{ffn}} \cdot \sigma(W_{\text{up}} \cdot h), \quad S_{\text{ffn}} = \text{diag}(s_1^{\text{ffn}}, \ldots, s_{d_{\text{ffn}}}^{\text{ffn}})
\end{equation}
where $h$ is the hidden state from the attention sublayer. After training, this is absorbed as $W_{\text{down}} \leftarrow W_{\text{down}} \cdot S_{\text{ffn}}$. We scale the $d_{\text{ffn}}$-dimensional intermediate rather than the $\dmodel$-dimensional output, because each intermediate dimension corresponds to a learned feature in the FFN. GAIN-FFN adds $d_{\text{ffn}} \times L$ parameters (184K for GPT-2, where $d_{\text{ffn}}\!=\!5120$), bringing the total to 230K.

\begin{center}\small
\begin{tabular}{llrr}
\toprule
Method & Structure & Params/layer & Total (GPT-2) \\
\midrule
GAIN & $\text{diag}(\dmodel)$ on $W_O$ & $\dmodel = 1280$ & 46,080 \\
GAIN-FFN & $\text{diag}$ on $W_O + W_{\text{down}}$ & $\dmodel + d_{\text{ffn}}$ & 230,400 \\
\midrule
LoRA ($r\!=\!1$) & $W + BA$ (additive) & $2r\dmodel$ & 92,160 \\
LoRA ($r\!=\!8$) & $W + BA$ (additive) & $2r\dmodel$ & 737,280 \\
\bottomrule
\end{tabular}
\end{center}
\noindent Both GAIN variants are multiplicative (Eq.~\ref{eq:principle},~\ref{eq:right_mult}) and absorbed after training for zero inference cost. LoRA is additive ($W + BA$).

% ============================================================
\section{Multiplicative Modulation}
\label{sec:theory}

We analyze why multiplicative modulation preserves pretrained capabilities while additive methods do not.

In transformers, each layer's output is added to a shared hidden state (the residual stream). The output projection $W_O$ determines what each attention layer contributes. LoRA modifies $W_O$ additively:
\begin{equation}
    \text{LoRA:} \quad W_{\text{new}} = W_O + BA \quad \text{(additive)}
\end{equation}
The perturbation $BA$ is independent of $W_O$, so for input $h$, the output $(W_O + BA)h = W_Oh + BAh$ includes $BAh$---a term that can introduce entirely new directions. GAIN modifies weights multiplicatively:
\begin{align}
    \text{GAIN:} \quad W_O^{\text{new}} &= S \cdot W_O \\
    \text{GAIN-FFN:} \quad W_{\text{down}}^{\text{new}} &= W_{\text{down}} \cdot S_{\text{ffn}}
\end{align}

\begin{proposition}[Output Space Preservation]
\label{prop:colspan}
For diagonal $S$ with $s_i > 0$ and weight matrix $W$:
\begin{enumerate}
    \item \textbf{Left multiplication} (GAIN on $W_O$): $\text{rowspace}(S \cdot W) = \text{rowspace}(W)$
    \item \textbf{Right multiplication} (GAIN-FFN on $W_{\text{down}}$): $\text{colspan}(W \cdot S) = \text{colspan}(W)$
\end{enumerate}
That is, multiplicative modulation can only produce outputs within the pretrained model's output subspace.
\end{proposition}

\begin{proof}
(1) The $i$-th row of $S \cdot W$ is $s_i \cdot w_i^{\top}$, a nonzero scaling of the $i$-th row of $W$. The row space is unchanged.
(2) The $j$-th column of $W \cdot S$ is $s_j \cdot w_j$, a nonzero scaling of the $j$-th column of $W$. The column space is unchanged.
\end{proof}

\noindent This property holds for any positive diagonal $S$, regardless of learning rate or training dynamics. LoRA has no analogous constraint: the perturbation $BA$ is independent of $W$, so the output space of $W + BA$ can escape that of $W$.

\begin{corollary}[Subspace overlap as a measurable proxy for forgetting]
\label{cor:subspace}
Proposition~\ref{prop:colspan} implies that for a multiplicative diagonal map $f(W)\!=\!S\!\cdot\!W$ with all $s_i\!>\!0$, the top-$k$ right-singular-vector subspace of $f(W)$ coincides with that of $W$ (the singular vectors are unchanged; only singular values are rescaled). For an additive map $f(W)\!=\!W+BA$ with $\text{rank}(BA)\!>\!0$ and $\text{col}(BA)\!\not\subseteq\!\text{col}(W)$, the top-$k$ subspace overlap is strictly less than $1$, and the deficit measures the magnitude of Biderman et al.\ \citeyear{biderman2024lora}'s ``intruder dimensions.''
\end{corollary}

\noindent Empirically on the sequential eight-domain protocol (Appendix~\ref{sec:mechanistic}, Table~\ref{tab:mechanistic}), the measured top-$10$ overlap is $\geq 0.88$ for multiplicative methods (GAIN, GAIN-FFN) and $\leq 0.56$ for additive methods (LoRA, DoRA). Lower overlap correlates monotonically with higher final forgetting, consistent with the subspace-preservation mechanism.

\citet{biderman2024lora} show that LoRA introduces ``intruder dimensions''---out-of-subspace singular vectors causally linked to forgetting. \citet{chen2024flatlora} show that LoRA minima can be sharp in full parameter space. GAIN sidesteps both: Proposition~\ref{prop:colspan} prevents intruder dimensions, and the cross-domain loss landscape remains flat (Appendix~\ref{sec:loss-landscape}).

Since Proposition~\ref{prop:colspan} holds at every layer, the guarantee composes through depth. For LoRA, intruder dimensions at early layers compound through subsequent layers.

\paragraph{Connection to neuroscience.} GAIN's principle has a direct analogue: \emph{gain modulation} in cortex \citep{salinas2000gain} scales neural responses multiplicatively, preserving stimulus selectivity while adapting amplitude. Equivalently, $h \!\cdot\! (S \!\cdot\! W_O) = (h \!\cdot\! S) \!\cdot\! W_O$: GAIN scales hidden activations rather than modifying the projection, exactly as gain fields scale neural firing rates without changing tuning curves (Appendix~\ref{sec:neuroscience}).

% ============================================================
\section{Single-Domain Adaptation}
\label{sec:experiments}

We first evaluate GAIN on single-domain adaptation: adapt a model to a new domain, then measure in-domain improvement and cross-domain damage. The sequential setting follows in \S\ref{sec:sequential}.

\textbf{Setup.} All experiments freeze the pretrained model and train only GAIN (or LoRA) parameters on a calibration set ($\sim$200K tokens, 5 epochs). We use PubMedQA medical abstracts as the primary adaptation domain because it represents a clear distributional shift from web pretraining data and has associated downstream benchmarks (MedQA, MMLU medical). We verify generality across additional domains in later sections. We measure forgetting via perplexity on held-out domains (LAMBADA, WikiText-103) and accuracy on general benchmarks (HellaSwag, ARC-Easy, PIQA, WinoGrande). Default learning rate is lr$=10^{-3}$.

We adapt GPT-2 (774M) on PubMedQA abstracts and measure both in-domain improvement (PubMed PPL) and cross-domain damage (WT-103 PPL, LAMBADA PPL, and four classification benchmarks). Table~\ref{tab:sweep} compares GAIN, GAIN-FFN, and LoRA across learning rates.

\begin{table}[t]
\centering
\caption{GPT-2 (774M) adapted on PubMedQA. PPL columns: \% change in perplexity ($\downarrow$ = better). Benchmark columns: accuracy change on general tasks. PubMed is in-domain; all others measure forgetting.}
\label{tab:sweep}
\resizebox{\textwidth}{!}{
\begin{tabular}{llr||rrr||rrrr}
\toprule
 & & & \multicolumn{3}{c||}{PPL $\Delta$\% ($\downarrow$ better)} & \multicolumn{4}{c}{Benchmark $\Delta$ (accuracy)} \\
Method & LR & Params & \textit{PubMed} & WT-103 & LAMB & ARC-E & HellaS & PIQA & WinoGr \\
\midrule
GAIN & $3\!\times\!10^{-4}$ & 46K & $-0.7$ & $+0.0$ & $-0.1$ & $+0.0$ & $+0.1$ & $+0.1$ & $-0.1$ \\
GAIN & $10^{-3}$ & 46K & $-2.0$ & $+0.1$ & $-0.4$ & $-0.1$ & $+0.1$ & $+0.1$ & $-0.1$ \\
\textbf{GAIN} & $\mathbf{3\!\times\!10^{-3}}$ & \textbf{46K} & $\mathbf{-4.2}$ & $\mathbf{+0.3}$ & $\mathbf{-0.9}$ & $-0.3$ & $+0.0$ & $+0.0$ & $-0.1$ \\
\midrule
GAIN-FFN & $10^{-3}$ & 230K & $-5.1$ & $+0.1$ & $-2.0$ & $-0.7$ & $-0.1$ & $+0.2$ & $+0.2$ \\
GAIN-FFN & $3\!\times\!10^{-3}$ & 230K & $-8.3$ & $+1.0$ & $-2.9$ & $-1.1$ & $-0.1$ & $+0.1$ & $+0.2$ \\
\midrule
LoRA & $10^{-4}$ & 92K & $-0.5$ & $+0.0$ & $-0.1$ & $+0.0$ & $+0.1$ & $+0.1$ & $+0.1$ \\
LoRA & $3\!\times\!10^{-4}$ & 92K & $-3.2$ & $+0.1$ & $-0.7$ & $-0.1$ & $+0.0$ & $+0.0$ & $-0.1$ \\
LoRA & $10^{-3}$ & 92K & $-7.4$ & $+1.1$ & $-3.0$ & $-1.2$ & $-0.4$ & $-0.1$ & $-0.2$ \\
LoRA & $10^{-2}$ & 92K & $-7.7$ & $\mathbf{+10.4}$ & $\mathbf{+7.1}$ & $\mathbf{-4.7}$ & $-0.7$ & $-0.8$ & $-0.2$ \\
\bottomrule
\end{tabular}}

{\small GAIN = $W_O$ scaling. GAIN-FFN = $W_O + W_{\text{down}}$ scaling. PPL: \% change. Benchmarks: accuracy $\Delta$.}
\end{table}

GAIN adapts with near-zero forgetting across its entire LR range. At lr$=3\!\times\!10^{-3}$, it achieves $-4.2\%$ PubMed PPL with all benchmarks within $\pm 0.3$, using \textbf{46K parameters}. GAIN-FFN adapts more aggressively ($-5.1\%$ at lr$=10^{-3}$, $-8.3\%$ at lr$=3\!\times\!10^{-3}$) with \textbf{230K parameters}, approaching LoRA's adaptation strength while maintaining near-zero forgetting. LoRA at its safe LR ($3\!\times\!10^{-4}$, 92K params) also adapts with minimal forgetting ($-3.2\%$), but damage grows with LR: at lr$=10^{-3}$, ARC-Easy drops by $1.2$ and LAMBADA PPL rises $3.0\%$; at lr$=10^{-2}$, ARC-Easy drops $4.7$ and WT-103 PPL rises $10.4\%$.

Single-domain adaptation results are consistent across model scales (GPT-2 774M through Llama-70B) and domains (WikiText, PG-19, PubMedQA); full tables are in the appendix. We also compare with full fine-tuning (catastrophic at 200K tokens) and (IA)$^3$ \citep{liu2022ia3} (avoids forgetting, consistent with our multiplicative framework); see appendix.

We also tested domain-specific classification tasks: GAIN improves MedQA accuracy by $+4.3$ and financial sentiment by $+12.0$ points without degrading general benchmarks, confirming that GAIN's benefits extend beyond perplexity. Full classification results are in the appendix. Additional comparisons: scaling across models (Appendix~\ref{sec:appendix-scaling}), full fine-tuning (Appendix~\ref{sec:full-ft}), and (IA)$^3$ (Appendix~\ref{sec:ia3-comparison}).

% ============================================================
\section{Sequential Multi-Domain Adaptation}
\label{sec:sequential}

In practice, models are adapted to new domains over time---a medical deployment followed by legal, then financial. Each new adaptation risks erasing earlier ones, a tension known as the stability-plasticity dilemma \citep{kirkpatrick2017overcoming}.

We emphasize that GAIN is not a continual learning method---it uses no replay, regularization, or task-specific modules. The absence of forgetting is a structural consequence of multiplicative modulation (Proposition~\ref{prop:colspan}), not an engineered solution to the stability-plasticity dilemma.

\subsection{Sequential Adaptation Across 8 Domains}

We train on eight domains sequentially (Medical $\to$ WikiText $\to$ Financial $\to$ PG-19 $\to$ MedQA $\to$ LAMBADA $\to$ SST-2 $\to$ HellaSwag), with 200K tokens and 200 steps $\times$ 5 epochs per domain. After training on each domain, we evaluate PPL on \emph{all} domains. Table~\ref{tab:continual} shows the full adaptation matrix. The \textcolor{gray}{upper triangle (gray)} measures forward interference---whether training on domain $i$ damages not-yet-trained domains. The \textbf{lower triangle} measures whether earlier-trained domains are preserved or forgotten. The \colorbox{blue!15}{diagonal} measures in-domain adaptation.

\begin{table}[h]
\centering\small
\caption{Sequential adaptation on GPT-2 Large (774M, 200K tokens/domain). Entries = \% PPL change vs.\ pretrained. \colorbox{blue!15}{Diagonal} = in-domain; lower triangle = preservation of earlier domains; \textcolor{gray}{upper triangle} = forward interference. See Tables~\ref{tab:continual_lr3e3}--\ref{tab:continual_lora_safe} for other learning rates and Table~\ref{tab:continual_mistral} for Mistral-7B.}
\label{tab:continual}
\begin{tabular}{l|rrrrrrrr}
\toprule
\multicolumn{9}{c}{\textbf{GAIN} (lr$=10^{-3}$, \% PPL change vs.\ pretrained baseline)} \\
\midrule
After training & Med & WT & Fin & PG & MdQ & LB & SST & HS \\
\midrule
+Med & \cellcolor{blue!15}$-2.9$ & \textcolor{gray}{$+0.0$} & \textcolor{gray}{$-0.0$} & \textcolor{gray}{$-1.3$} & \textcolor{gray}{$-0.3$} & \textcolor{gray}{$-0.6$} & \textcolor{gray}{$-0.0$} & \textcolor{gray}{$-0.2$} \\
+WT & $-3.0$ & \cellcolor{blue!15}$-11.0$ & \textcolor{gray}{$-0.1$} & \textcolor{gray}{$-3.8$} & \textcolor{gray}{$-0.2$} & \textcolor{gray}{$-1.5$} & \textcolor{gray}{$-2.1$} & \textcolor{gray}{$-0.3$} \\
+Fin & $-2.9$ & $-11.1$ & \cellcolor{blue!15}$-5.6$ & \textcolor{gray}{$-2.9$} & \textcolor{gray}{$+0.0$} & \textcolor{gray}{$-2.6$} & \textcolor{gray}{$-1.6$} & \textcolor{gray}{$-0.4$} \\
+PG & $-2.8$ & $-11.0$ & $-5.7$ & \cellcolor{blue!15}$-24.0$ & \textcolor{gray}{$-0.1$} & \textcolor{gray}{$-4.3$} & \textcolor{gray}{$-1.7$} & \textcolor{gray}{$-0.2$} \\
+MdQ & $-2.4$ & $-10.6$ & $-5.4$ & $-23.9$ & \cellcolor{blue!15}$-8.0$ & \textcolor{gray}{$-5.0$} & \textcolor{gray}{$-1.4$} & \textcolor{gray}{$-0.3$} \\
+LB & $-2.2$ & $-10.0$ & $-5.2$ & $-23.8$ & $-8.1$ & \cellcolor{blue!15}$-14.2$ & \textcolor{gray}{$-0.7$} & \textcolor{gray}{$+0.5$} \\
+SST & $-2.1$ & $-10.6$ & $-4.8$ & $-23.6$ & $-7.8$ & $-14.0$ & \cellcolor{blue!15}$-11.9$ & \textcolor{gray}{$+0.5$} \\
+HS & $-2.2$ & $-10.3$ & $-4.6$ & $-23.5$ & $-7.5$ & $-13.5$ & $-11.7$ & \cellcolor{blue!15}$-8.5$ \\
\midrule
\multicolumn{9}{c}{\textbf{LoRA} (lr$=10^{-3}$, \% PPL change vs.\ pretrained baseline)} \\
\midrule
+Med & \cellcolor{blue!15}$-9.8$ & \textcolor{gray}{$+1.0$} & \textcolor{gray}{$+2.5$} & \textcolor{gray}{$-16.2$} & \textcolor{gray}{$+3.1$} & \textcolor{gray}{$-2.4$} & \textcolor{gray}{$+2.5$} & \textcolor{gray}{$+0.6$} \\
+WT & $+32.5$ & \cellcolor{blue!15}$-25.2$ & \textcolor{gray}{$+16.3$} & \textcolor{gray}{$-4.7$} & \textcolor{gray}{$+16.2$} & \textcolor{gray}{$+18.0$} & \textcolor{gray}{$-1.2$} & \textcolor{gray}{$+8.7$} \\
+Fin & $+7.2$ & $-8.4$ & \cellcolor{blue!15}$-13.4$ & \textcolor{gray}{$-9.1$} & \textcolor{gray}{$+3.9$} & \textcolor{gray}{$-4.2$} & \textcolor{gray}{$+3.3$} & \textcolor{gray}{$+3.0$} \\
+PG & $+4.7$ & $-9.8$ & $-4.2$ & \cellcolor{blue!15}$-25.5$ & \textcolor{gray}{$+3.7$} & \textcolor{gray}{$-0.9$} & \textcolor{gray}{$-0.3$} & \textcolor{gray}{$+3.4$} \\
+MdQ & $+15.0$ & $+1.2$ & $+6.4$ & $-20.9$ & \cellcolor{blue!15}$-23.9$ & \textcolor{gray}{$-0.9$} & \textcolor{gray}{$+3.6$} & \textcolor{gray}{$+6.9$} \\
+LB & $+14.5$ & $+14.5$ & $+15.5$ & $-15.4$ & $-6.8$ & \cellcolor{blue!15}$-28.1$ & \textcolor{gray}{$+12.3$} & \textcolor{gray}{$+10.0$} \\
+SST & $+30.7$ & $+23.6$ & $+11.6$ & $-4.1$ & $+0.2$ & $+5.6$ & \cellcolor{blue!15}$-24.2$ & \textcolor{gray}{$+24.5$} \\
+HS & $+22.6$ & $+31.0$ & $\mathbf{+37.4}$ & $-4.5$ & $+15.0$ & $+22.5$ & $+1.6$ & \cellcolor{blue!15}$-20.8$ \\
\bottomrule
\end{tabular}
\end{table}

\paragraph{Upper triangle: why is interference near zero?} GAIN's gray entries are uniformly small---training on one domain barely affects predictions on others. This may seem too good to be true. We verify it is not a bug: GAIN's learned scaling factors satisfy $s_{\min} = 0.94$ and $s_{\max} = 1.06$ (GPT-2 Large, lr$=10^{-3}$), with $98.7\%$ of dimensions perturbed by less than $\pm 5\%$. Per-token analysis (Figure~\ref{fig:token_adj}) confirms that each token's cross-entropy changes by $\sigma \approx 0.12$ nats, with improvements and worsenings nearly canceling. LoRA's distribution is $3\times$ wider ($\sigma \approx 0.30$ nats) with heavy tails that produce the large positive entries in its upper triangle ($+18\%$, $+16\%$).

\paragraph{Lower triangle: earlier domains are preserved.} GAIN preserves each domain's improvement through all subsequent trainings. PG-19 improves by $24.0\%$ at step 4 and retains $23.5\%$ after 4 more domains. Every lower-triangle entry shows improvement: adaptations accumulate rather than interfere. LoRA's improvements are steadily eroded: Medical improves by $9.8\%$ at step 1 but degrades by $22.6\%$ after eight domains---the original gain is not just lost but reversed. Even at LoRA's safe learning rate (lr$=3\!\times\!10^{-4}$), degradation still accumulates to $10.2\%$ after eight domains (Table~\ref{tab:continual_lora_safe}).

Forgetting and interference are both manifestations of \emph{feature departure}: additive perturbation $W + BA$ introduces directions outside the pretrained output space, distorting the shared representation for all domains, whether trained before or after. Proposition~\ref{prop:colspan} prevents this and bounds the perturbation: for diagonal $S$ with $s_i > 0$, the singular values satisfy $s_{\min} \cdot \sigma_i \leq \sigma_i(SW) \leq s_{\max} \cdot \sigma_i$. No new output directions enter, and existing ones change by at most a factor of $s_{\max}/s_{\min} = 1.13$. LoRA has no such bound.

\paragraph{Why does positive transfer occur?} GAIN not only avoids forgetting---it \emph{improves} earlier domains. To understand this, we train GAIN-FFN independently on each of five domains (five separate models, each starting from pretrained) and compare the resulting scaling factors. Most adaptation is domain-specific: cross-domain correlations are near zero ($r = 0.00$--$0.14$). However, $20$--$40\%$ of the variance is shared across all domains---all five agree on which dimensions to suppress and amplify. We hypothesize this shared component corrects pretrained feature weightings that are suboptimal for downstream text generally; training on any domain applies this correction, benefiting all others.

\paragraph{Causal verification.} We decompose each domain's scaling update into shared and domain-specific components. Removing the shared component halves positive transfer. Keeping \emph{only} the shared component retains $81\%$ of it. The shared component is thus both necessary and largely sufficient for positive transfer.

\begin{figure}[t]
\centering
\includegraphics[width=0.95\textwidth]{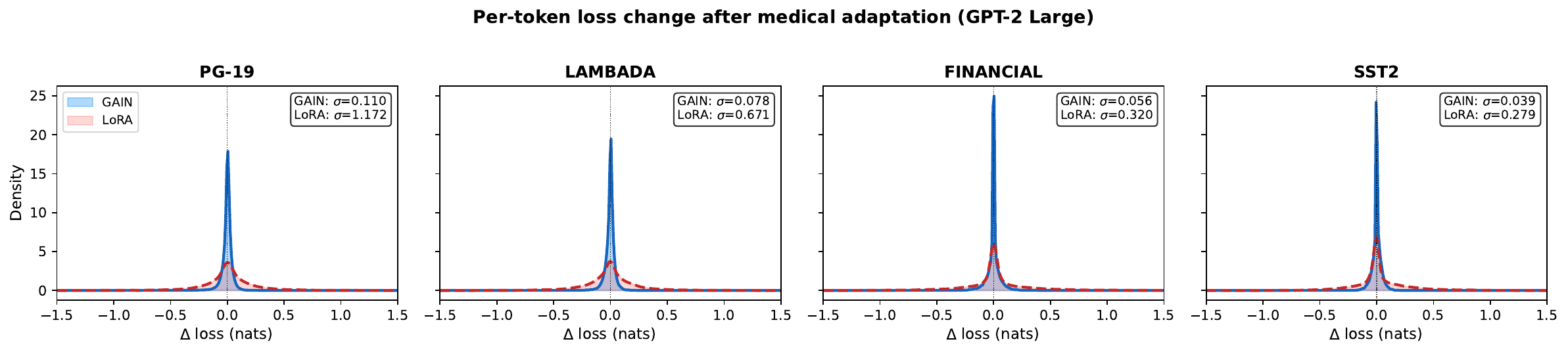}
\caption{Per-token loss change on four unrelated domains after medical adaptation (GPT-2 Large). GAIN ($\sigma \approx 0.12$) is $3\times$ narrower than LoRA ($\sigma \approx 0.30$).}
\label{fig:token_adj}
\end{figure}

\subsection{Robustness}

We verify the main result under extensive variation of hyperparameters, data, baselines, and model scales. The result is robust across:
\begin{itemize}\setlength{\itemsep}{1pt}\setlength{\parskip}{0pt}
    \item \textbf{Learning rates:} GAIN preserves from lr$=10^{-3}$ to $3\!\times\!10^{-3}$; LoRA forgets at every LR including its safe $3\!\times\!10^{-4}$ (Table~\ref{tab:continual_lora_safe}).
    \item \textbf{Domain orderings:} GAIN-FFN $13.3 \pm 1.1\%$ improvement across eight random orderings on GPT-2, $8.5 \pm 0.7\%$ on Mistral-7B.
    \item \textbf{Random seeds:} GAIN-FFN $-13.2 \pm 0.2\%$; LoRA $r\!=\!8$: $+11.9 \pm 0.5\%$ (three seeds each).
    \item \textbf{Data scale:} At $25\!\times$ the main-text token budget, GAIN is unchanged while LoRA worsens to $+15.7\%$ (Appendix~\ref{sec:larger-data}).
    \item \textbf{19 distinct domains:} GAIN-FFN's gap over LoRA $r\!=\!8$ remains $\geq 5\!\times$ (Appendix~\ref{sec:extended-19}).
    \item \textbf{OOD:} English-only GPT-2 adapted to German Wikipedia---GAIN-FFN exhibits $45\!\times$ less forward interference (Appendix~\ref{sec:ood-test}).
    \item \textbf{24 sequential adaptations:} Scaling factors remain bounded ($s_{\min}\!=\!0.43$, $s_{\max}\!=\!1.57$; Appendix~\ref{sec:long-seq-trajectory}).
    \item \textbf{LoRA at matched params:} LoRA $r\!=\!2$ (184K, comparable to GAIN-FFN's 230K) still degrades by $14.5\%$.
    \item \textbf{CL baselines:} GAIN-FFN matches Replay-LoRA, dominates EWC on the Pareto front, beats PackNet by $17$ points (Appendices~\ref{sec:replay-comparison}, \ref{sec:ewc-comparison}, \ref{sec:packnet-comparison}).
    \item \textbf{Model scales} (five models, 774M to 70B):
\end{itemize}

\begin{center}\small
\begin{tabular}{llrr}
\toprule
Model & Family & GAIN-FFN & LoRA \\
\midrule
GPT-2 Large (774M) & OpenAI & $\mathbf{12.9\%}$ \textbf{better} & $17.9\%$ worse \\
Mistral-7B & Mistral AI & $\mathbf{7.0\%}$ \textbf{better} & $28.0\%$ worse \\
Qwen2.5-7B & Alibaba & $\mathbf{6.6\%}$ \textbf{better} & $36.1\%$ worse \\
Llama-2-13B & Meta & $\mathbf{8.8\%}$ \textbf{better} & $24.0\%$ worse \\
Llama-2-70B & Meta & $\mathbf{8.1\%}$ \textbf{better} & $19.5\%$ worse \\
\bottomrule
\multicolumn{4}{l}{\small Avg PPL change on earlier-trained domains after eight sequential adaptations.}
\end{tabular}
\end{center}

Across five models from four families (774M to 70B), GAIN-FFN improves earlier-trained domains while LoRA degrades them. The gap widens on instruction-tuned models: on Llama-3-8B-Instruct, GAIN-FFN improves by $20.6\%$ while LoRA degrades by $56.9\%$ (Appendix~\ref{sec:llama-instruct}). Direct subspace measurement confirms the mechanism: LoRA's adapted weight retains only $0.53$ top-10 subspace overlap with the pretrained $W$, while GAIN-FFN preserves $0.91$ (Appendix~\ref{sec:mechanistic}).

\subsection{LoRA's forgetting-adaptation tradeoff}

LoRA's forgetting can be reduced via conservative learning rates or L2 regularization ($\ell_{\text{total}} = \ell_{\text{LM}} + \lambda \sum_i (p_i - p_i^{\text{prev}})^2$, where $p^{\text{prev}}$ are parameters saved after the previous domain). The principled extension---weighting that L2 penalty by the diagonal Fisher (EWC, \citealp{kirkpatrick2017overcoming})---traces out the same kind of frontier (Appendix~\ref{sec:ewc-comparison}, Table~\ref{tab:ewc}). But this always sacrifices in-domain adaptation (Figure~\ref{fig:tradeoff}). L2, EWC, and replay all require \emph{domain boundaries} (to snapshot parameters, Fisher, or buffer-size policies); GAIN needs no such signal.

\begin{figure}[t]
\centering
\includegraphics[width=0.55\textwidth]{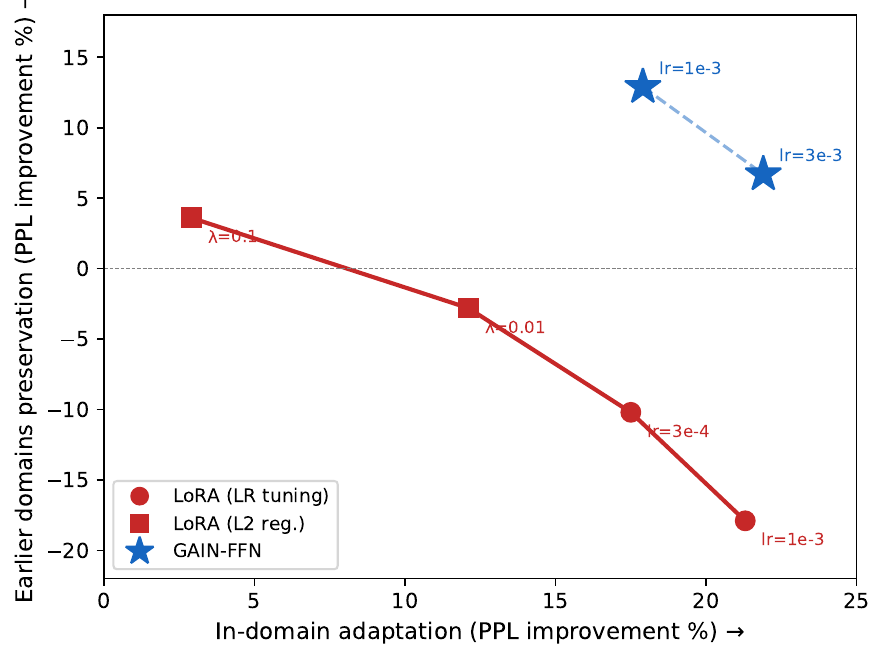}
\caption{LoRA's forgetting-adaptation tradeoff. Red points are LoRA with different learning rates (lr) and L2 regularization strengths ($\lambda$, where $\ell = \ell_{\text{LM}} + \lambda \|p - p^{\text{prev}}\|^2$). Higher $\lambda$ reduces forgetting but sacrifices adaptation. GAIN-FFN (blue star) operates outside this tradeoff curve at lr $\leq 3\!\times\!10^{-3}$.}
\label{fig:tradeoff}
\end{figure}

\subsection{GAIN-FFN breaks the tradeoff}

Table~\ref{tab:continual} shows GAIN with $W_O$ scaling only (46K params), which adapts conservatively. Extending to also scale the FFN output (GAIN-FFN, 230K params) yields much stronger in-domain adaptation while maintaining positive transfer:

\begin{center}\small
\begin{tabular}{lrrrr}
\toprule
Method & Params & Steps & In-domain $\downarrow$ & Earlier domains \\
\midrule
GAIN ($W_O$) & 46K & 200 & $-10.8\%$ & $10.5\%$ better \\
\textbf{GAIN-FFN} & \textbf{230K} & \textbf{200} & $\mathbf{-17.9\%}$ & $\mathbf{12.9\%}$ \textbf{better} \\
\midrule
LoRA ($r\!=\!1$, lr$=10^{-3}$) & 92K & 200 & $-21.4\%$ & $17.9\%$ worse \\
LoRA ($r\!=\!2$, lr$=10^{-3}$) & 184K & 200 & $-22.6\%$ & $14.5\%$ worse \\
LoRA ($r\!=\!1$, lr$=3\!\times\!10^{-4}$) & 92K & 1000 & $-17.6\%$ & $7.2\%$ worse \\
\bottomrule
\multicolumn{5}{l}{\small In-domain: avg diagonal PPL change ($\downarrow$ = better adaptation). Earlier domains:} \\
\multicolumn{5}{l}{\small avg PPL change on previously trained domains. GAIN-FFN: $13.3 \pm 1.1\%$ over eight orderings.}
\end{tabular}
\end{center}

GAIN-FFN operates above LoRA's Pareto frontier (Figure~\ref{fig:tradeoff}): at 200 steps, it achieves $17.9\%$ in-domain improvement, comparable to LoRA's safe-LR result ($17.6\%$) which requires $5\times$ more training steps (1000 vs 200). Yet GAIN-FFN improves earlier domains by $12.9\%$ while LoRA degrades them by $7.2\%$---even at LoRA's safest configuration. LoRA $r\!=\!2$ (184K params, comparable to GAIN-FFN's 230K) still degrades by $14.5\%$, confirming that the issue is additive vs multiplicative, not parameter count. We further verify this for the standard transformer-LoRA ranks ($r\!=\!8$, $r\!=\!16$) and target sets (QKVO) in Appendix~\ref{sec:lora_standard_ranks} (Table~\ref{tab:lora_ranks}): no rank or target combination tested yields positive transfer, and QKVO targets in fact \emph{worsen} forgetting (up to $+27.4\%$ at $r\!=\!16$). The multiplicative guarantee (Proposition~\ref{prop:colspan}) holds for both $W_O$ and $W_{\text{down}}$ independently, so extending to more weight matrices does not compromise the structural protection against forgetting.

\subsection{Downstream benchmarks after sequential adaptation}

Does sequential adaptation damage general capabilities? We train GAIN-FFN and LoRA sequentially on seven domains (excluding HellaSwag to avoid benchmark contamination), then evaluate on seven benchmarks not seen during training:

\begin{table}[h]
\centering\small
\caption{Benchmark accuracy change after seven sequential domain adaptations (acc\_norm). Same hyperparameters across all models (lr$=10^{-3}$, 200 steps, 5 epochs per domain). Benchmarks are not in the training set.}
\label{tab:downstream}
\resizebox{\textwidth}{!}{
\begin{tabular}{l|rrr|rrr|rrr|rrr|rrr}
\toprule
 & \multicolumn{3}{c|}{GPT-2 (774M)} & \multicolumn{3}{c|}{Mistral-7B} & \multicolumn{3}{c|}{Qwen2.5-7B} & \multicolumn{3}{c|}{Llama-13B} & \multicolumn{3}{c}{Llama-70B} \\
 & Base & GAIN & LoRA & Base & GAIN & LoRA & Base & GAIN & LoRA & Base & GAIN & LoRA & Base & GAIN & LoRA \\
\midrule
ARC-E & $46.6$ & $-1.4$ & $-1.9$ & $79.5$ & $+1.5$ & $\mathbf{-10.7}$ & $77.4$ & $\mathbf{+3.8}$ & $-0.2$ & $76.4$ & $+0.5$ & $+1.5$ & $79.5$ & $\mathbf{+1.4}$ & $+0.2$ \\
ARC-C & $25.1$ & $-1.0$ & $-0.1$ & $54.3$ & $+0.3$ & $\mathbf{-10.0}$ & $51.2$ & $\mathbf{+3.1}$ & $-0.1$ & $49.0$ & $+0.0$ & $+2.5$ & $57.0$ & $\mathbf{+1.4}$ & $-1.0$ \\
BoolQ & $60.5$ & $-2.1$ & $\mathbf{-7.8}$ & $84.3$ & $-1.2$ & $\mathbf{-9.2}$ & $84.3$ & $-0.8$ & $\mathbf{-14.9}$ & $82.5$ & $-1.4$ & $-2.3$ & $85.3$ & $-2.5$ & $\mathbf{-14.6}$ \\
HellaS & $45.3$ & $-0.9$ & $-2.1$ & $81.2$ & $+0.1$ & $\mathbf{-4.1}$ & $78.9$ & $+0.4$ & $-0.9$ & $79.7$ & $+0.9$ & $\mathbf{-5.3}$ & $84.3$ & $\mathbf{+0.7}$ & $\mathbf{-7.1}$ \\
OBQA & $31.2$ & $-1.8$ & $-2.2$ & $46.2$ & $+0.0$ & $\mathbf{-3.0}$ & $47.4$ & $-1.0$ & $\mathbf{-2.8}$ & $45.6$ & $+0.4$ & $\mathbf{-3.2}$ & $49.0$ & $-0.2$ & $\mathbf{-10.6}$ \\
PIQA & $69.2$ & $-1.3$ & $\mathbf{-2.6}$ & $82.8$ & $-0.1$ & $\mathbf{-5.2}$ & $79.9$ & $+0.2$ & $-1.5$ & $80.6$ & $+0.0$ & $-1.7$ & $82.5$ & $-0.2$ & $\mathbf{-5.8}$ \\
WinoGr & $55.3$ & $+0.4$ & $+0.3$ & $75.1$ & $-0.4$ & $\mathbf{-3.4}$ & $72.5$ & $+0.7$ & $\mathbf{-2.8}$ & $72.2$ & $+1.0$ & $-2.6$ & $80.3$ & $-1.1$ & $\mathbf{-9.1}$ \\
\bottomrule
\end{tabular}}
\end{table}

On GPT-2, GAIN-FFN's changes are within $\pm 2\%$; LoRA's BoolQ drops $7.8\%$. On larger models, the pattern is more dramatic: GAIN-FFN \emph{improves} multiple benchmarks (Qwen ARC-Easy $+3.8\%$, ARC-Challenge $+3.1\%$; Llama-70B ARC-Easy $+1.4\%$, HellaSwag $+0.7\%$), while LoRA damages most benchmarks by $2$--$15\%$. BoolQ is particularly vulnerable to LoRA: $-14.9\%$ on Qwen. The pattern holds from 774M to 70B with identical hyperparameters (lr$=10^{-3}$, 200 steps, 5 epochs). General capabilities are preserved---and often enhanced---by multiplicative adaptation.

% ============================================================
\section{Conclusion}

We show that domain adaptation in pretrained LLMs should be multiplicative, not additive. Scaling existing representations preserves model capabilities; adding new directions causes forgetting. The theoretical basis is Proposition~\ref{prop:colspan}: multiplicative modulation stays within the pretrained output space, bounding perturbation to a factor of $s_{\max}/s_{\min}$.

The strongest evidence comes from sequential domain adaptation: across eight domains, five models from four families (774M to 70B), and multiple random domain orderings, GAIN-FFN \emph{improves} earlier-trained domains by $7$--$13\%$ while LoRA degrades them by $18$--$36\%$. Downstream benchmarks confirm: after seven sequential adaptations on Mistral-7B, GAIN-FFN improves ARC-Easy by $1.5\%$ while LoRA degrades it by $10.7\%$. Forgetting and forward interference are both caused by feature departure---additive methods leaving the output space---and multiplicative modulation prevents both by construction.

\textbf{Limitations.}

\emph{Capacity ceiling.} GAIN's diagonal scaling can only reweight existing pretrained features---it cannot synthesise genuinely new ones. For domains that are far out-of-distribution for the pretrained model, GAIN will eventually hit a capacity ceiling. We probe this boundary by adapting the English-only-pretrained GPT-2 Large to German Wikipedia (Appendix~\ref{sec:ood-test}): GAIN-FFN still improves in-domain German PPL by $7.4\%$, only $3$ points behind LoRA, suggesting the capacity ceiling is not reached for German. For genuinely new modalities (e.g., adapting an English text model to an unseen language with disjoint vocabulary, or to structured/tabular data with no pretraining signal), GAIN may truly plateau and additive methods may become necessary; this more extreme regime remains empirically untested in our paper.

\emph{Domain diversity.} Our nineteen evaluation domains are all English text drawn from sources broadly similar in structure to common pretraining corpora (news, scientific articles, code, dialogue, reviews); the dichotomy may attenuate when the pretrained representation space is genuinely insufficient for the new domain.

\emph{Sequence depth.} We test up to 24 sequential adaptations; scaling-factor drift (Appendix~\ref{sec:long-seq-trajectory}) is bounded ($s_{\min}\!=\!0.43$, $s_{\max}\!=\!1.57$ at step 24) but very long sequences ($\gg 24$) may require periodic re-normalisation or lower learning rates.

\emph{Theory--practice gap.} Proposition~\ref{prop:colspan} is a per-layer linear-algebraic statement; deep transformers compose nonlinear layers, so the colspan guarantee carries over to the network output empirically rather than provably. Our experiments confirm the per-layer guarantee suffices to prevent forgetting in practice, but a theorem covering the full nonlinear stack remains open.

Within these scope conditions, our results suggest that adaptation should scale existing representations rather than adding new ones.

% ============================================================
\bibliography{brainstorm}

\appendix

\section{Related Work}

LoRA and variants \citep{hu2022lora, dettmers2023qlora, liu2024dora} add low-rank perturbations to weight matrices. We show these additive methods live on a forgetting--adaptation tradeoff curve that multiplicative methods avoid. \citet{biderman2024lora} identify ``intruder dimensions'' in LoRA that interfere with pretrained capabilities---our Proposition~\ref{prop:colspan} shows GAIN avoids these by construction. \citet{chen2024flatlora} show that LoRA minima can be deceptively sharp in full parameter space, causing forgetting.

(IA)$^3$ \citep{liu2022ia3} scales keys, values, and FFN activations multiplicatively for few-shot classification. HiRA \citep{zhang2025hira} uses Hadamard products of pretrained weights with low-rank matrices. SVFT \citep{lingam2024svft} scales singular values of pretrained weights. PiCa \citep{wang2025pica} projects gradients onto the pretrained column space. All are multiplicative but none studies forgetting or sequential adaptation. We confirm that (IA)$^3$ also exhibits positive transfer in the sequential setting ($-14.2\%$, comparable to GAIN-FFN's $-12.9\%$), validating that the multiplicative principle---not GAIN's specific parameterization---is the key mechanism (Appendix~\ref{sec:ia3-comparison}). \textbf{What is new beyond prior multiplicative PEFT.} The mathematical fact that diagonal scaling preserves a column space is elementary; what is new here is (a) the formal characterisation of forgetting via colspan invariance (Proposition~\ref{prop:colspan}, building on the empirical intruder-dimensions observation of \citet{biderman2024lora}), (b) the empirical demonstration that this dichotomy predicts the gap between every additive and every multiplicative adaptation method we test, including unrelated ones like (IA)$^3$, and (c) GAIN-FFN as the simplest sufficient parameterisation---scaling only the attention output and FFN-output projections rather than internal activations as in (IA)$^3$ or singular values as in SVFT.

Continual learning methods address forgetting through replay buffers \citep{rolnick2019experience}, regularization \citep{kirkpatrick2017overcoming}, or architecture partitioning \citep{mallya2018packnet}. GAIN is not a continual learning method---its protection is structural (Proposition~\ref{prop:colspan}), not engineered.

\section{Sequential Adaptation on Mistral-7B}

Table~\ref{tab:continual_mistral} replicates the sequential adaptation experiment (Table~\ref{tab:continual}) on Mistral-7B with domain order WikiText $\to$ Medical $\to$ Financial $\to$ PG-19 $\to$ MedQA $\to$ LAMBADA $\to$ SST-2 $\to$ HellaSwag. The pattern is consistent: GAIN preserves all earlier domains while LoRA accumulates catastrophic forgetting.

\begin{table}[h]
\centering
\caption{Sequential adaptation on Mistral-7B (200K tokens per domain). Same format as Table~\ref{tab:continual}.}
\label{tab:continual_mistral}
\begin{tabular}{l|rrrrrrrr}
\toprule
\multicolumn{9}{c}{\textbf{GAIN} (lr$=10^{-3}$, \% PPL change vs.\ pretrained baseline)} \\
\midrule
After training & WT & Med & Fin & PG & MdQ & LB & SST & HS \\
\midrule
+WT & \cellcolor{blue!15}$-23.1$ \\
+Med & $-23.2$ & \cellcolor{blue!15}$-1.2$ \\
+Fin & $-23.1$ & $-1.1$ & \cellcolor{blue!15}$-5.0$ \\
+PG & $-22.9$ & $-1.1$ & $-5.0$ & \cellcolor{blue!15}$-0.3$ \\
+MdQ & $-22.7$ & $-1.1$ & $-4.9$ & $-0.2$ & \cellcolor{blue!15}$-4.6$ \\
+LB & $-22.6$ & $-1.0$ & $-5.0$ & $+0.0$ & $-4.5$ & \cellcolor{blue!15}$-6.5$ \\
+SST & $-22.4$ & $-0.9$ & $-4.9$ & $+0.2$ & $-4.4$ & $-6.4$ & \cellcolor{blue!15}$-6.9$ \\
+HS & $-22.1$ & $-0.8$ & $-4.7$ & $+0.3$ & $-4.2$ & $-6.2$ & $-6.9$ & \cellcolor{blue!15}$-9.3$ \\
\midrule
\multicolumn{9}{c}{\textbf{LoRA} (lr$=10^{-3}$, \% PPL change vs.\ pretrained baseline)} \\
\midrule
+WT & \cellcolor{blue!15}$-30.9$ \\
+Med & $-10.8$ & \cellcolor{blue!15}$+0.1$ \\
+Fin & $-1.7$ & $+6.3$ & \cellcolor{blue!15}$-7.8$ \\
+PG & $+5.7$ & $+6.7$ & $+1.3$ & \cellcolor{blue!15}$+11.2$ \\
+MdQ & $+10.9$ & $+9.7$ & $+7.4$ & $+16.1$ & \cellcolor{blue!15}$-8.7$ \\
+LB & $+15.3$ & $+11.0$ & $+13.7$ & $+17.1$ & $+7.4$ & \cellcolor{blue!15}$-8.6$ \\
+SST & $+43.6$ & $+20.0$ & $+48.0$ & $+39.6$ & $+17.9$ & $\mathbf{+78.5}$ & \cellcolor{blue!15}$-15.4$ \\
+HS & $+44.4$ & $+21.8$ & $\mathbf{+50.6}$ & $+52.0$ & $+19.3$ & $+57.6$ & $+18.6$ & \cellcolor{blue!15}$-19.0$ \\
\bottomrule
\end{tabular}
\end{table}

\section{Sequential Adaptation at Higher Learning Rate}

Table~\ref{tab:continual_lr3e3} shows GAIN at lr$=3\!\times\!10^{-3}$ on GPT-2 Large. In-domain adaptation is nearly $2\times$ stronger than lr$=10^{-3}$ (e.g., WikiText $-15.8\%$ vs $-11.0\%$), while earlier domains still improve by $11.4\%$ on average. Minor forward interference appears on later domains (HellaSwag $+4.2\%$ in the upper triangle) and the medical adaptation erodes from $-5.7\%$ to $-0.7\%$. This is explained by the wider scaling range: at lr$=3\!\times\!10^{-3}$, $s_{\min} = 0.83$ and $s_{\max} = 1.19$ (vs.\ $0.94$--$1.06$ at lr$=10^{-3}$), with only $55.6\%$ of dimensions within $[0.95, 1.05]$ (vs.\ $98.7\%$). The perturbation bound is looser, but still bounded---all scaling factors remain within $[0.8, 1.2]$, far tighter than LoRA's unbounded perturbation.

\begin{table}[h]
\centering\small
\caption{Sequential adaptation on GPT-2 Large at lr$=3\!\times\!10^{-3}$ (GAIN). Same format as Table~\ref{tab:continual}. Stronger adaptation with minor interference on later domains.}
\label{tab:continual_lr3e3}
\begin{tabular}{l|rrrrrrrr}
\toprule
\multicolumn{9}{c}{\textbf{GAIN} (lr$=3\!\times\!10^{-3}$, \% PPL change vs.\ pretrained baseline)} \\
\midrule
After training & Med & WT & Fin & PG & MdQ & LB & SST & HS \\
\midrule
+Med & \cellcolor{blue!15}$-5.7$ & \textcolor{gray}{$+0.2$} & \textcolor{gray}{$+0.0$} & \textcolor{gray}{$-3.5$} & \textcolor{gray}{$-0.3$} & \textcolor{gray}{$-1.3$} & \textcolor{gray}{$+0.1$} & \textcolor{gray}{$-0.2$} \\
+WT & $-4.9$ & \cellcolor{blue!15}$-15.8$ & \textcolor{gray}{$+0.5$} & \textcolor{gray}{$-6.4$} & \textcolor{gray}{$+0.2$} & \textcolor{gray}{$-1.6$} & \textcolor{gray}{$-2.7$} & \textcolor{gray}{$+0.2$} \\
+Fin & $-3.9$ & $-15.3$ & \cellcolor{blue!15}$-8.6$ & \textcolor{gray}{$-4.4$} & \textcolor{gray}{$+0.7$} & \textcolor{gray}{$-3.6$} & \textcolor{gray}{$-1.0$} & \textcolor{gray}{$+0.5$} \\
+PG & $-3.5$ & $-14.5$ & $-8.2$ & \cellcolor{blue!15}$-25.0$ & \textcolor{gray}{$+0.5$} & \textcolor{gray}{$-5.0$} & \textcolor{gray}{$-1.0$} & \textcolor{gray}{$+0.8$} \\
+MdQ & $-2.1$ & $-13.4$ & $-7.2$ & $-24.2$ & \cellcolor{blue!15}$-14.2$ & \textcolor{gray}{$-5.5$} & \textcolor{gray}{$-0.2$} & \textcolor{gray}{$+1.5$} \\
+LB & $-1.3$ & $-11.7$ & $-6.0$ & $-23.0$ & $-13.3$ & \cellcolor{blue!15}$-20.4$ & \textcolor{gray}{$+1.2$} & \textcolor{gray}{$+3.8$} \\
+SST & $-0.6$ & $-11.6$ & $-4.5$ & $-22.0$ & $-12.2$ & $-19.3$ & \cellcolor{blue!15}$-17.1$ & \textcolor{gray}{$+4.2$} \\
+HS & $-0.7$ & $-10.5$ & $-3.9$ & $-21.0$ & $-10.7$ & $-17.3$ & $-16.0$ & \cellcolor{blue!15}$-13.8$ \\
\bottomrule
\end{tabular}
\end{table}

\section{LoRA at Safe Learning Rate}

Table~\ref{tab:continual_lora_safe} shows LoRA at lr$=3\!\times\!10^{-4}$ (the safe LR from Table~\ref{tab:sweep} where single-domain forgetting is near zero). Even at this conservative LR, sequential adaptation accumulates $+10.2\%$ average forgetting after eight domains. LoRA's forgetting in sequential settings is not a learning rate tuning issue---it is structural.

\begin{table}[h]
\centering
\caption{Sequential adaptation: LoRA at safe lr$=3\!\times\!10^{-4}$ on GPT-2 Large. Same format as Table~\ref{tab:continual}. Even at the safe LR, forgetting accumulates to $+10.2\%$.}
\label{tab:continual_lora_safe}
\begin{tabular}{l|rrrrrrrr}
\toprule
\multicolumn{9}{c}{\textbf{LoRA} (lr$=3\!\times\!10^{-4}$, \% PPL change vs.\ pretrained baseline)} \\
\midrule
After training & Med & WT & Fin & PG & MdQ & LB & SST & HS \\
\midrule
+Med & \cellcolor{blue!15}$-5.6$ & \textcolor{gray}{$+0.2$} & \textcolor{gray}{$+0.1$} & \textcolor{gray}{$-3.5$} & \textcolor{gray}{$-0.3$} & \textcolor{gray}{$-1.3$} & \textcolor{gray}{$+0.2$} & \textcolor{gray}{$-0.3$} \\
+WT & $+3.8$ & \cellcolor{blue!15}$-19.7$ & \textcolor{gray}{$+2.4$} & \textcolor{gray}{$-5.6$} & \textcolor{gray}{$+4.7$} & \textcolor{gray}{$+2.8$} & \textcolor{gray}{$-2.5$} & \textcolor{gray}{$+2.4$} \\
+Fin & $+3.0$ & $-13.9$ & \cellcolor{blue!15}$-9.1$ & \textcolor{gray}{$-12.6$} & \textcolor{gray}{$+2.5$} & \textcolor{gray}{$-4.9$} & \textcolor{gray}{$+0.6$} & \textcolor{gray}{$+1.1$} \\
+PG & $+1.5$ & $-11.7$ & $-6.8$ & \cellcolor{blue!15}$-25.5$ & \textcolor{gray}{$+1.9$} & \textcolor{gray}{$-5.4$} & \textcolor{gray}{$-1.2$} & \textcolor{gray}{$+1.4$} \\
+MdQ & $+9.3$ & $-5.1$ & $-1.0$ & $-21.9$ & \cellcolor{blue!15}$-17.3$ & \textcolor{gray}{$-3.8$} & \textcolor{gray}{$+1.4$} & \textcolor{gray}{$+2.7$} \\
+LB & $+5.4$ & $+1.3$ & $+4.3$ & $-19.2$ & $-7.1$ & \cellcolor{blue!15}$-24.6$ & \textcolor{gray}{$+8.9$} & \textcolor{gray}{$+4.9$} \\
+SST & $+10.9$ & $+5.7$ & $+9.8$ & $-6.9$ & $+2.4$ & $-7.1$ & \cellcolor{blue!15}$-21.3$ & \textcolor{gray}{$+12.0$} \\
+HS & $+15.9$ & $+8.5$ & $+26.4$ & $-2.0$ & $+16.9$ & $+3.2$ & $-3.9$ & \cellcolor{blue!15}$-17.0$ \\
\bottomrule
\end{tabular}
\end{table}

\section{LoRA at Standard Ranks and QKVO Targets}
\label{sec:lora_standard_ranks}

Standard practice trains LoRA with rank $8$ or $16$, often on all four attention projections ($Q$, $K$, $V$, $O$). Table~\ref{tab:continual} reports LoRA at $r\!=\!1$; Table~\ref{tab:lora_ranks} extends this to $r\!=\!8$ and $r\!=\!16$ on $W_O$ and on the QKVO target set. We reuse the sequential eight-domain protocol of Section~\ref{sec:sequential}, with the default domain order of the released script (WT $\to$ Med $\to$ Fin $\to$ PG $\to$ MdQ $\to$ LB $\to$ SST $\to$ HS; numbers within Table~\ref{tab:lora_ranks} are directly comparable).

\begin{table}[h]
\centering\small
\caption{LoRA at standard ranks and QKVO target sets, on GPT-2 Large. Same protocol as Table~\ref{tab:continual} ($5$ epochs, $200$ steps, $200$K tokens/domain, lr$=10^{-3}$ unless noted). \textit{In-domain}: mean PPL change on the diagonal ($\downarrow$ = better adaptation). \textit{Final forgetting}: mean PPL change on the seven earlier-trained domains after the last adaptation (positive $=$ forgetting).}
\label{tab:lora_ranks}
\begin{tabular}{llrrrr}
\toprule
Method & Targets & LR & Params & In-domain $\downarrow$ & Final forgetting \\
\midrule
\textbf{GAIN ($W_O$)} & --- & $10^{-3}$ & \textbf{46K} & $-10.2\%$ & $\mathbf{-9.8\%}$ \\
\midrule
LoRA $r\!=\!1$ & $W_O$ & $10^{-3}$ & 92K & $-21.3\%$ & $+12.2\%$ \\
LoRA $r\!=\!8$ & $W_O$ & $10^{-3}$ & 737K & $-24.3\%$ & $+11.7\%$ \\
LoRA $r\!=\!16$ & $W_O$ & $10^{-3}$ & 1.47M & $-24.5\%$ & $+13.3\%$ \\
LoRA $r\!=\!8$ & $W_O$ & $3\!\times\!10^{-4}$ & 737K & $-21.5\%$ & $+14.0\%$ \\
\midrule
LoRA $r\!=\!8$ & QKVO & $10^{-3}$ & 2.21M & $-23.8\%$ & $+21.6\%$ \\
LoRA $r\!=\!16$ & QKVO & $10^{-3}$ & 4.42M & $-21.4\%$ & $+27.4\%$ \\
LoRA $r\!=\!8$ & QKVO & $3\!\times\!10^{-4}$ & 2.21M & $-24.0\%$ & $+15.8\%$ \\
\bottomrule
\end{tabular}
\end{table}

Three findings:
\begin{itemize}
\item \textbf{Forgetting is not a low-rank artefact.} On $W_O$, LoRA's forgetting is non-monotonic in rank: $+12.2\%$ at $r\!=\!1$, $+11.7\%$ at $r\!=\!8$, $+13.3\%$ at $r\!=\!16$. Higher rank does not buy stability.
\item \textbf{Wider targets worsen forgetting.} Applying LoRA to QKVO rather than $W_O$ alone increases forgetting at every rank and LR: $+21.6\%$ vs $+11.7\%$ at $r\!=\!8$, $+27.4\%$ vs $+13.3\%$ at $r\!=\!16$. The safe LR ($3\!\times\!10^{-4}$) narrows but does not close the gap ($+15.8\%$).
\item \textbf{Parameter budget does not explain the gap.} GAIN ($46$K parameters) retains $9.8\%$ improvement on earlier domains after eight adaptations. LoRA $r\!=\!16$ on QKVO ($4.42$M parameters, $96\times$ more) forgets $27.4\%$.
\end{itemize}

Within the transformer-LoRA design space, no rank or target combination we tested yields positive transfer. The mechanism is unchanged by rank or target selection: additive $W + BA$ introduces directions outside the pretrained output subspace, violating the invariance guaranteed by multiplicative modulation (Proposition~\ref{prop:colspan}). Increasing rank or widening targets increases the off-subspace component rather than eliminating it.

\section{Loss Landscape Analysis}
\label{sec:loss-landscape}

Figure~\ref{fig:landscape} linearly interpolates between pretrained ($\alpha\!=\!0$) and adapted ($\alpha\!=\!1$) weights on Mistral-7B: $W(\alpha) = (1-\alpha) W_{\text{pre}} + \alpha W_{\text{adapted}}$. Both methods reduce in-domain PPL, but GAIN's cross-domain loss stays flat (LAMBADA PPL $13.90$--$13.92$) while LoRA's rises steeply ($13.90 \to 16.21$). The cross-domain loss variation across the full interpolation path is $0.02$ for GAIN vs $2.31$ for LoRA ($115\!\times$ larger).

\begin{figure}[h]
\centering
\includegraphics[width=0.95\textwidth]{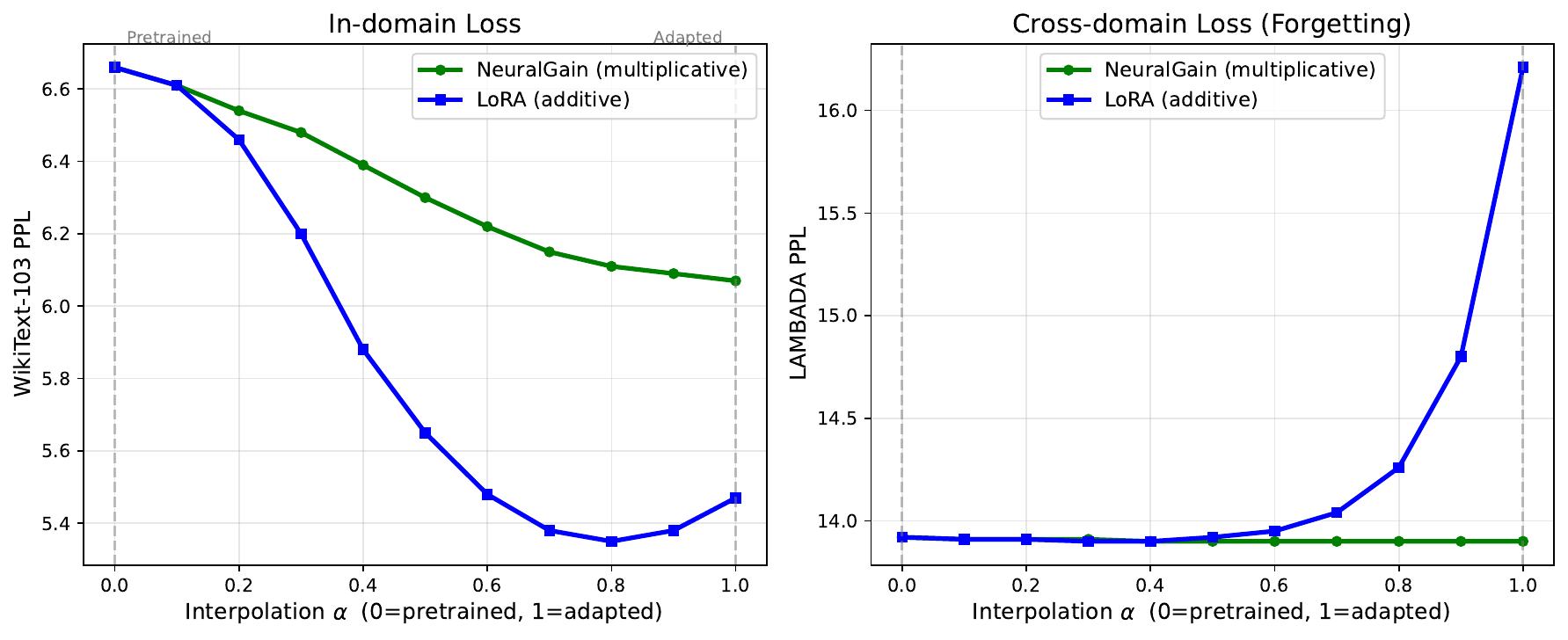}
\caption{Loss landscape interpolation on Mistral-7B. Left: in-domain PPL decreases for both. Right: GAIN's cross-domain loss is flat; LoRA's rises steeply.}
\label{fig:landscape}
\end{figure}

This flatness extends to multi-domain combinations. We train GAIN and LoRA separately on WikiText and Medical text, then evaluate on a 2D grid combining both adaptation directions. GAIN's domain directions are nearly orthogonal (cosine similarity $0.002$) vs.\ LoRA ($0.033$, $16\times$ higher). Figure~\ref{fig:landscape2d} shows the result: GAIN's cross-domain loss is constant regardless of how domains are combined (range $0.007$ in log PPL), while LoRA's varies $885\times$ more.

\begin{figure}[h]
\centering
\includegraphics[width=0.95\textwidth]{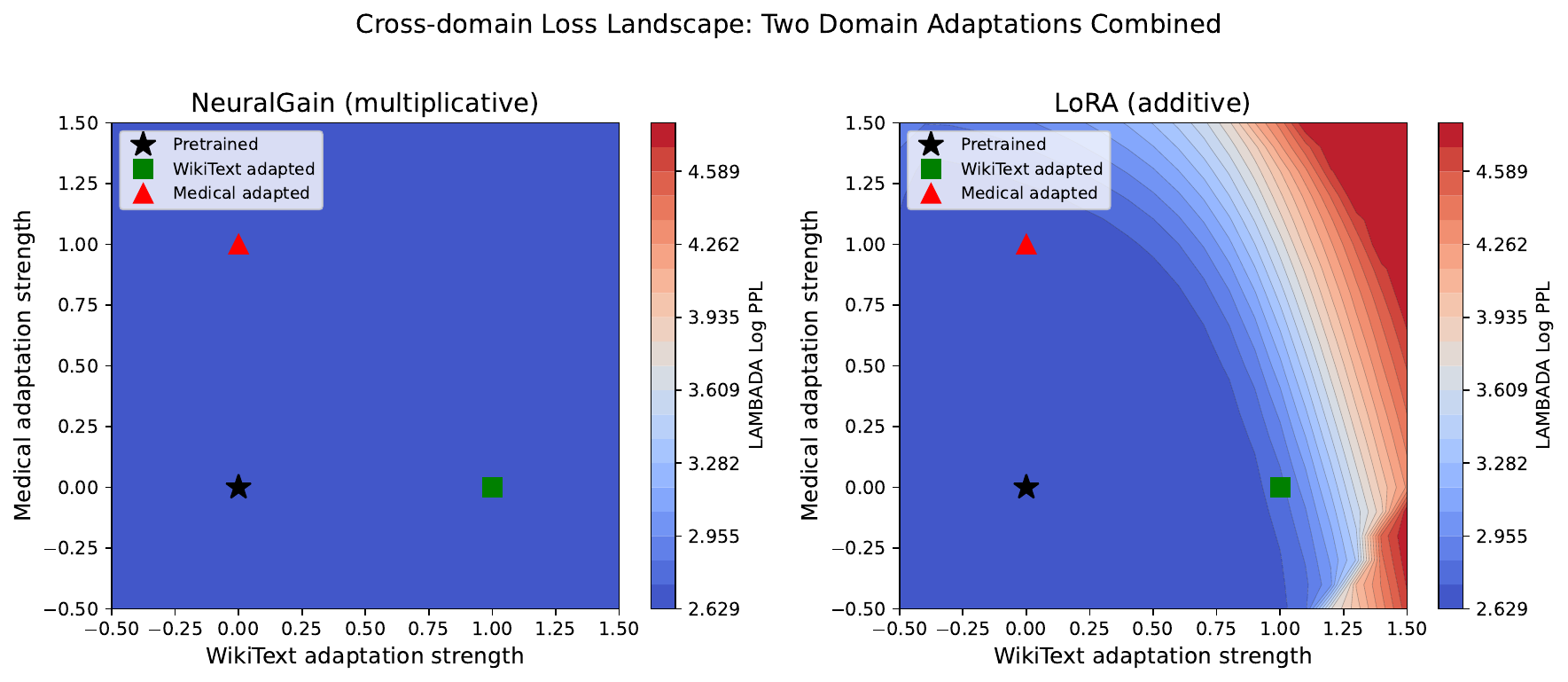}
\caption{Cross-domain loss when WikiText and Medical adaptations are combined. \textbf{Left}: GAIN---uniform blue; any combination preserves LAMBADA. \textbf{Right}: LoRA---red regions appear when adaptations are combined.}
\label{fig:landscape2d}
\end{figure}

\section{Single-Domain Adaptation Across Model Scales}
\label{sec:appendix-scaling}

Table~\ref{tab:scaling} shows single-domain adaptation (PubMedQA) across four models from two families (774M to 70B). GAIN adapts with near-zero cross-domain damage at every scale. LoRA at lr$=10^{-3}$ damages LAMBADA increasingly with model size; even at its safe LR ($3\!\times\!10^{-4}$), cross-domain damage remains.

\begin{table}[h]
\centering
\caption{Single-domain adaptation (PubMedQA) across model scales. PPL columns: \% change ($\downarrow$ = better). PubMed is in-domain; LAMBADA measures forgetting. Benchmarks: accuracy (4 tasks averaged).}
\label{tab:scaling}
\resizebox{\textwidth}{!}{
\begin{tabular}{llr||rr||rrrr}
\toprule
 & & & \multicolumn{2}{c||}{PPL $\Delta$\%} & \multicolumn{4}{c}{Benchmark $\Delta$ (accuracy)} \\
Model & Method & Params & \textit{PubMed} & LAMB & ARC-E & HellaS & PIQA & WinoGr \\
\midrule
\multirow{3}{*}{Mistral-7B} & GAIN & 131K & $-3.1$ & $-0.4$ & $-0.3$ & $-0.2$ & $-0.5$ & $+0.0$ \\
 & LoRA (lr$=10^{-3}$) & 262K & $+5.4$ & $+2.7$ & $+0.6$ & $+0.3$ & $-0.3$ & $-0.1$ \\
 & LoRA (lr$=3\!\times\!10^{-4}$) & 262K & $-3.7$ & $-0.0$ & $+0.4$ & $+0.1$ & $-0.3$ & $-0.1$ \\
\midrule
\multirow{3}{*}{Llama-2-13B} & GAIN & 205K & $-2.4$ & $-0.3$ & $-0.1$ & $+0.4$ & $+0.0$ & $+0.5$ \\
 & LoRA (lr$=10^{-3}$) & 410K & $+1.4$ & $-0.2$ & $-1.3$ & $+0.9$ & $+0.2$ & $-0.5$ \\
 & LoRA (lr$=3\!\times\!10^{-4}$) & 410K & $-4.0$ & $-0.2$ & $-0.7$ & $+0.3$ & $+0.0$ & $-0.1$ \\
\midrule
\multirow{3}{*}{Llama-2-70B} & GAIN & 655K & $-1.3$ & $-0.1$ & $+0.0$ & $+0.1$ & $+0.2$ & $+0.8$ \\
 & LoRA (lr$=10^{-3}$) & 1.3M & $+8.3$ & $+0.5$ & $-0.9$ & $+0.9$ & $+0.0$ & $-0.3$ \\
 & LoRA (lr$=3\!\times\!10^{-4}$) & 1.3M & $-4.1$ & $-0.5$ & $-0.6$ & $+0.3$ & $+0.0$ & $+0.9$ \\
\bottomrule
\end{tabular}}
\end{table}

GAIN achieves in-domain adaptation ($-1.3\%$ to $-3.1\%$) with near-zero cross-domain damage (LAMBADA within $\pm 0.4\%$) and benchmark changes within $\pm 0.5$ at every scale from 7B to 70B. LoRA at its default LR damages LAMBADA by $+0.5\%$ to $+2.7\%$ and PubMed by $+1.4\%$ to $+8.3\%$ on two of three models; at safe LR, forgetting is controlled but adaptation is modest.

\textbf{Domain generality.} On Llama-2-70B adapted to PG-19 (fiction, a different domain from PubMedQA):

\begin{center}
\begin{tabular}{llrr}
\toprule
Method & LR & PG-19 $\Delta$PPL & LAMBADA $\Delta$PPL \\
\midrule
GAIN & $10^{-3}$ & $-0.9\%$ & $-0.1\%$ \\
LoRA & $10^{-3}$ & $+17.5\%$ & $+20.9\%$ \\
\bottomrule
\end{tabular}
\end{center}

The same pattern holds on a different domain: GAIN adapts with zero forgetting while LoRA causes catastrophic cross-domain damage ($+20.9\%$ LAMBADA).

\section{GAIN vs LoRA at Matched Learning Rates}
\label{sec:lr-matched}

Table~\ref{tab:sweep} compares GAIN and LoRA at their respective default LRs. To ensure a fair comparison at comparable aggressiveness, we also test GAIN at lr$=10^{-2}$ and LoRA at lr$=3\!\times\!10^{-3}$ (both more aggressive than their defaults):

\begin{table}[h]
\centering
\caption{GAIN (lr$=10^{-2}$) vs LoRA (lr$=3\!\times\!10^{-3}$) on WikiText adaptation. GAIN adapts strongly with zero forgetting at every scale. LoRA diverges on Mistral-7B and Llama-2-70B, and causes $+10\%$ forgetting on Llama-2-13B.}
\label{tab:lr-matched}
\begin{tabular}{lrrrr}
\toprule
 & \multicolumn{2}{c}{GAIN (lr$=10^{-2}$)} & \multicolumn{2}{c}{LoRA (lr$=3\!\times\!10^{-3}$)} \\
\cmidrule(lr){2-3} \cmidrule(lr){4-5}
Model & WikiText & LAMBADA & WikiText & LAMBADA \\
\midrule
Mistral-7B & $\mathbf{-15.9\%}$ & $+0.2\%$ & $+213\%$ & $+180\%$ \\
Llama-2-13B & $\mathbf{-25.7\%}$ & $-0.1\%$ & $-28.4\%$ & $+10.0\%$ \\
Llama-2-70B & $\mathbf{-18.0\%}$ & $+0.1\%$ & $+12.3\%$ & $+35.6\%$ \\
\bottomrule
\end{tabular}
\end{table}

LoRA at lr$=3\!\times\!10^{-3}$ diverges catastrophically on two of three models. On Llama-2-13B where LoRA does not diverge, it achieves slightly better in-domain adaptation ($-28.4\%$ vs $-25.7\%$) but at the cost of $+10\%$ LAMBADA forgetting. GAIN at lr$=10^{-2}$ is uniformly safe: strong adaptation with LAMBADA within $\pm 0.2\%$ across all scales.

\section{Domain-Specific Classification Tasks}
\label{sec:appendix-classification}

GAIN's benefits extend beyond perplexity to classification accuracy. We train on domain-specific data and evaluate on domain-specific benchmarks.

\textbf{MedQA.} GAIN trained on PubMedQA (classification loss, lr$=3\!\times\!10^{-3}$) improves MedQA 4-option accuracy from 51.3 to 55.6 ($+4.3$ points) on Mistral-7B, with all general benchmarks within $\pm 0.8$:

\begin{center}\small
\begin{tabular}{lrrrrrr}
\toprule
Method & MedQA & ARC-E & HellaS & LAMB & PIQA & WinoGr \\
\midrule
Baseline & 51.3 & 79.5 & 81.2 & 74.9 & 82.8 & 75.1 \\
GAIN & $\mathbf{55.6}$ & $+0.8$ & $-0.2$ & $-0.1$ & $-0.1$ & $-0.3$ \\
LoRA (lr$=3\!\times\!10^{-4}$) & 55.9 & $-1.5$ & $-0.5$ & $-0.3$ & $-1.5$ & $-0.9$ \\
LoRA (lr$=10^{-3}$) & 28.6 & $-11.5$ & $-1.0$ & $+0.3$ & $+0.0$ & $-0.1$ \\
\bottomrule
\end{tabular}
\end{center}

GAIN and LoRA (safe LR) achieve comparable MedQA improvement, but LoRA's general benchmark damage is $2$--$3\times$ larger. At LoRA's default LR ($10^{-3}$), MedQA collapses to 28.6 ($-22.7$ points).

\textbf{Financial sentiment.} GAIN trained on financial news sentiment (lr$=10^{-3}$) improves financial sentiment accuracy from 56.4\% to 68.4\% ($+12.0$ points) and SST-2 from 55.7 to 62.4 ($+6.7$ points), with all general benchmarks within $\pm 1.0$:

\begin{center}\small
\begin{tabular}{lrrrrr}
\toprule
Method & Financial & SST-2 & ARC-E & HellaS & WinoGr \\
\midrule
Baseline & 56.4 & 55.7 & 79.5 & 81.2 & 75.1 \\
GAIN & $\mathbf{68.4}$ ($+12.0$) & $62.4$ ($+6.7$) & $+0.7$ & $+0.2$ & $+1.0$ \\
LoRA (lr$=3\!\times\!10^{-4}$) & $66.6$ ($+10.2$) & $71.0$ ($+15.3$) & $+1.3$ & $-0.2$ & $-1.1$ \\
LoRA (lr$=10^{-3}$) & $\mathbf{75.4}$ ($+19.0$) & $88.3$ ($+32.6$) & $-1.6$ & $-2.4$ & $-1.3$ \\
\bottomrule
\end{tabular}
\end{center}

LoRA at lr$=10^{-3}$ achieves stronger task-specific adaptation ($+19.0$) but at the cost of $1$--$2\%$ degradation on general benchmarks. GAIN provides moderate adaptation ($+12.0$) with zero degradation.

\section{Full Fine-Tuning Comparison}
\label{sec:full-ft}

Full fine-tuning on $\sim$200K tokens is catastrophic:

\begin{center}
\begin{tabular}{lrr}
\toprule
Model & In-domain $\Delta$PPL & LAMBADA $\Delta$PPL \\
\midrule
GPT-2 Large & $-24.1\%$ & $+16.2\%$ \\
Mistral-7B & $+62.0\%$ & $+220.5\%$ \\
\bottomrule
\end{tabular}
\end{center}

On GPT-2, full fine-tuning achieves strong in-domain adaptation ($-24.1\%$) but damages LAMBADA ($+16.2\%$). On Mistral-7B, 200K tokens is far too little data for full fine-tuning---the model diverges ($+62\%$ in-domain, $+220\%$ cross-domain). GAIN avoids both failure modes through its multiplicative constraint.

\section{GAIN-FFN Learning Rate Sweep}
\label{sec:lr-sweep}

Table~\ref{tab:lr-sweep} shows how GAIN-FFN's adaptation--preservation balance varies with learning rate on GPT-2 Large (sequential 8-domain adaptation). At lr$=10^{-3}$ (paper default), GAIN-FFN adapts moderately with strong positive transfer. At lr$=3\!\times\!10^{-3}$, in-domain adaptation reaches $-21.9\%$, matching LoRA's $-21.4\%$ at lr$=10^{-3}$, while earlier domains still improve by $6.7\%$. At lr$=10^{-2}$, training destabilizes: in-domain adaptation is strongest ($-22.5\%$) but earlier domains degrade by $+13.5\%$.

This degradation is \emph{not} forgetting in the LoRA sense (no intruder dimensions are introduced). Strictly, Proposition~\ref{prop:colspan} holds even for negative $s_i$---a negative scaling factor reverses a dimension's sign but does not escape the output space. However, sign-flipping a hidden dimension effectively inverts the feature that downstream layers have been calibrated for, causing optimization instability rather than structural departure. Table~\ref{tab:scaling-stats} tracks the scaling factor distribution across domains:

\begin{table}[h]
\centering\small
\caption{Scaling factor statistics after each sequential domain (GPT-2 Large). At lr$=10^{-3}$, factors stay positive and tightly concentrated. At lr$=3\!\times\!10^{-3}$, factors go negative by domain 5. At lr$=10^{-2}$, factors go negative by domain 2 and reach $-1.33$ by domain 8. Negative scaling factors sign-flip hidden dimensions, destroying the multiplicative guarantee.}
\label{tab:scaling-stats}
\begin{tabular}{l|rrr|rrr|rrr}
\toprule
 & \multicolumn{3}{c|}{lr$=10^{-3}$} & \multicolumn{3}{c|}{lr$=3\!\times\!10^{-3}$} & \multicolumn{3}{c}{lr$=10^{-2}$} \\
After & $s_{\min}$ & $s_{\max}$ & \% tight & $s_{\min}$ & $s_{\max}$ & \% tight & $s_{\min}$ & $s_{\max}$ & \% tight \\
\midrule
+Dom 1 & $0.94$ & $1.06$ & $99.9$ & $0.84$ & $1.16$ & $67.6$ & $0.52$ & $1.48$ & $29.3$ \\
+Dom 2 & $0.78$ & $1.25$ & $73.3$ & $0.40$ & $1.62$ & $38.1$ & $\mathbf{-0.76}$ & $2.36$ & $16.6$ \\
+Dom 4 & $0.61$ & $1.40$ & $55.8$ & $0.07$ & $1.94$ & $27.8$ & $\mathbf{-0.83}$ & $2.87$ & $12.9$ \\
+Dom 6 & $0.34$ & $1.78$ & $42.7$ & $\mathbf{-0.38}$ & $2.15$ & $22.3$ & $\mathbf{-1.25}$ & $3.06$ & $11.5$ \\
+Dom 8 & $0.24$ & $1.77$ & $36.8$ & $\mathbf{-0.51}$ & $2.28$ & $19.3$ & $\mathbf{-1.33}$ & $3.11$ & $10.4$ \\
\bottomrule
\multicolumn{10}{l}{\small \% tight = percentage of scaling factors in $[0.95, 1.05]$.}
\end{tabular}
\end{table}

At lr$=10^{-3}$, all scaling factors remain positive ($s_{\min} = 0.24$) through all eight domains, and the output space preservation guarantee holds. At lr$=3\!\times\!10^{-3}$, some factors go negative by domain 5 ($s_{\min} = -0.17$), sign-flipping hidden dimensions---this is not the injection of new directions (as in LoRA) but the inversion of existing ones, a form of training instability. At lr$=10^{-2}$, factors go negative by domain 2 and reach $s_{\min} = -1.33$ by domain 8. The practical safe range for GAIN-FFN is lr$\leq 10^{-3}$ (all factors positive) to lr$=3\!\times\!10^{-3}$ (mild instability on late domains).

\begin{table}[h]
\centering\small
\caption{GAIN-FFN learning rate sweep on GPT-2 Large (sequential 8-domain adaptation, 200K tokens/domain). In-domain: avg diagonal PPL improvement. Earlier domains: avg PPL change on previously trained domains (negative = improvement).}
\label{tab:lr-sweep}
\begin{tabular}{lrrr}
\toprule
LR & In-domain $\downarrow$ & Earlier domains & Regime \\
\midrule
$5\!\times\!10^{-4}$ & $-14.6\%$ & $\mathbf{-12.8\%}$ \textbf{better} & conservative \\
$10^{-3}$ & $-17.9\%$ & $\mathbf{-12.9\%}$ \textbf{better} & default \\
$2\!\times\!10^{-3}$ & $-20.7\%$ & $\mathbf{-10.1\%}$ \textbf{better} & aggressive \\
$3\!\times\!10^{-3}$ & $-21.9\%$ & $\mathbf{-6.7\%}$ \textbf{better} & matches LoRA in-domain \\
$10^{-2}$ & $-22.5\%$ & $+13.5\%$ worse & unstable \\
\midrule
LoRA ($10^{-3}$) & $-21.4\%$ & $+17.9\%$ worse & --- \\
\bottomrule
\end{tabular}
\end{table}

GAIN-FFN maintains positive transfer across a wide learning rate range ($5\!\times\!10^{-4}$ to $3\!\times\!10^{-3}$, a $6\times$ range). At lr$=3\!\times\!10^{-3}$, it matches LoRA's in-domain adaptation while still \emph{improving} earlier domains---the opposite of LoRA's $17.9\%$ degradation. Training becomes unstable only at lr$=10^{-2}$, where the diagonal constraint is no longer sufficient to keep perturbations small.

\textbf{Does forward interference accumulate?} At lr$=10^{-3}$, forward interference remains bounded through all eight domains (Table~\ref{tab:scaling-stats}: all $s_i > 0$, \% tight decreases gradually). The scaling factor trajectory suggests interference is bounded by the learning rate, not by the number of domains. For very long domain sequences ($\gg 8$), lower learning rates or periodic re-normalisation of $S$ toward identity could maintain the guarantee. Empirically, the guarantee survives up to $24$ sequential adaptations (Appendix~\ref{sec:long-seq-trajectory}, Table~\ref{tab:long-seq-trajectory}; $s_{\min}\!=\!0.43$ at step 24) and up to $19$ distinct domains (Appendix~\ref{sec:extended-19}); evaluation at $\gg 30$ domains remains future work.

\section{Scaling to Larger Data Budgets}
\label{sec:larger-data}

Our main experiments use 200K tokens per domain. To test whether the results hold at realistic data scales, we repeat the sequential 8-domain experiment with $10\times$ more data (2M tokens per domain, 200 steps $\times$ 5 epochs).

\begin{table}[h]
\centering\small
\caption{Sequential adaptation with 2M tokens per domain (GPT-2 Large). Same format as Table~\ref{tab:continual}. GAIN-FFN maintains positive transfer; LoRA's forgetting worsens.}
\label{tab:2M}
\begin{tabular}{l|rrrrrrrr}
\toprule
\multicolumn{9}{c}{\textbf{GAIN-FFN} (2M tokens, lr$=10^{-3}$, \% PPL change vs.\ pretrained)} \\
\midrule
After training & Med & WT & Fin & PG & MdQ & LB & SST & HS \\
\midrule
+Med & \cellcolor{blue!15}$-6.8$ & $+0.5$ & $-0.3$ & $-3.1$ & $-0.1$ & $-2.7$ & $+0.5$ & $-0.4$ \\
+WT & $-4.1$ & \cellcolor{blue!15}$-17.6$ & $+0.0$ & $-2.8$ & $+0.3$ & $-2.5$ & $-3.7$ & $+0.7$ \\
+Fin & $-2.9$ & $-15.9$ & \cellcolor{blue!15}$-11.7$ & $-3.0$ & $-0.7$ & $-7.7$ & $-1.0$ & $+0.2$ \\
+PG & $-3.1$ & $-14.7$ & $-10.9$ & \cellcolor{blue!15}$-12.6$ & $-1.0$ & $-8.3$ & $-1.5$ & $+0.9$ \\
+MdQ & $-0.7$ & $-12.2$ & $-9.7$ & $-9.9$ & \cellcolor{blue!15}$-23.1$ & $-9.1$ & $+1.4$ & $+2.3$ \\
+LB & $+0.0$ & $-10.9$ & $-8.0$ & $-7.1$ & $-20.9$ & \cellcolor{blue!15}$-25.5$ & $+3.0$ & $+3.9$ \\
+SST & $+1.0$ & $-11.0$ & $-5.7$ & $-4.3$ & $-19.3$ & $-23.1$ & \cellcolor{blue!15}$-22.9$ & $+6.4$ \\
+HS & $+1.4$ & $-8.6$ & $-4.6$ & $-2.3$ & $-16.9$ & $-20.7$ & $-19.9$ & \cellcolor{blue!15}$-16.4$ \\
\midrule
\multicolumn{9}{c}{\textbf{LoRA} (2M tokens, lr$=10^{-3}$, \% PPL change vs.\ pretrained)} \\
\midrule
+Med & \cellcolor{blue!15}$-9.7$ & $+2.1$ & $+0.9$ & $-5.2$ & $+3.7$ & $-3.4$ & $+2.4$ & $+0.5$ \\
+WT & $+32.2$ & \cellcolor{blue!15}$-24.0$ & $+14.6$ & $+3.5$ & $+13.0$ & $+24.4$ & $+1.4$ & $+13.0$ \\
+Fin & $+6.3$ & $-6.5$ & \cellcolor{blue!15}$-13.5$ & $+0.8$ & $+2.9$ & $-5.4$ & $+3.1$ & $+3.0$ \\
+PG & $+3.5$ & $-8.7$ & $-4.8$ & \cellcolor{blue!15}$-12.9$ & $+1.5$ & $-2.1$ & $+0.9$ & $+2.8$ \\
+MdQ & $+12.8$ & $+4.1$ & $+5.7$ & $-4.6$ & \cellcolor{blue!15}$-26.0$ & $-0.9$ & $+4.0$ & $+7.4$ \\
+LB & $+14.1$ & $+7.1$ & $+8.9$ & $+1.5$ & $-3.2$ & \cellcolor{blue!15}$-29.5$ & $+10.6$ & $+11.7$ \\
+SST & $+42.8$ & $+21.7$ & $+22.5$ & $+12.2$ & $+13.8$ & $+1.6$ & \cellcolor{blue!15}$-26.2$ & $+40.1$ \\
+HS & $+22.0$ & $+22.1$ & $+34.9$ & $+9.2$ & $+19.1$ & $+14.9$ & $+2.2$ & \cellcolor{blue!15}$-21.0$ \\
\bottomrule
\end{tabular}
\end{table}

At 2M tokens, GAIN-FFN maintains positive transfer ($-10.2\%$ avg improvement on earlier domains) while LoRA's forgetting is comparable to the 200K setting ($+17.8\%$ avg degradation). We replicate this on Mistral-7B: GAIN-FFN shows $-7.2\%$ positive transfer while LoRA degrades by $+45.5\%$---even worse than at 200K ($+28.0\%$). The pattern holds across data scales and model sizes:

\begin{center}\small
\begin{tabular}{lllrr}
\toprule
Model & Method & Tokens & In-domain $\downarrow$ & Earlier domains \\
\midrule
\multirow{4}{*}{GPT-2 (774M)} & GAIN-FFN & 200K & $-17.9\%$ & $\mathbf{-12.9\%}$ \textbf{better} \\
 & GAIN-FFN & 2M & $-17.1\%$ & $\mathbf{-10.2\%}$ \textbf{better} \\
 & LoRA & 200K & $-21.4\%$ & $+17.9\%$ worse \\
 & LoRA & 2M & $-20.4\%$ & $+17.8\%$ worse \\
\midrule
\multirow{4}{*}{Mistral-7B} & GAIN-FFN & 200K & $-7.0\%$ & $\mathbf{-7.0\%}$ \textbf{better} \\
 & GAIN-FFN & 2M & $-9.8\%$ & $\mathbf{-7.2\%}$ \textbf{better} \\
 & LoRA & 200K & $-14.0\%$ & $+28.0\%$ worse \\
 & LoRA & 2M & $-13.2\%$ & $+45.5\%$ worse \\
\bottomrule
\end{tabular}
\end{center}

GAIN-FFN's positive transfer is slightly reduced at larger data budgets because more training drives scaling factors further from unity, but remains solidly positive at both scales. LoRA's forgetting is unchanged on GPT-2 and \emph{worsens} on Mistral-7B at 2M tokens---the intruder dimension problem is structural, not data-dependent, and compounds with stronger adaptation.

\paragraph{Further scaling to 5M tokens per domain.} To probe a wider range, we also repeat the protocol with $5$M tokens per domain ($25\times$ our main-text budget), directly comparing GAIN ($W_O$-only, $46$K parameters) and LoRA $r\!=\!8$ on $W_O$ ($737$K parameters), using the same sequential protocol as Table~\ref{tab:lora_ranks}. Results (Table~\ref{tab:scale5m}) confirm the same pattern: GAIN's behaviour is essentially unchanged as the token budget grows from $200$K to $5$M, while LoRA's final forgetting \emph{worsens} from $+11.7\%$ to $+15.7\%$. More data does not close the gap.

\begin{table}[h]
\centering\small
\caption{Sequential 8-domain adaptation at $5$M tokens/domain on GPT-2 Large, same hyperparameters as Table~\ref{tab:lora_ranks} ($5$ epochs, $200$ steps, lr$=10^{-3}$). \textit{Final forgetting}: mean PPL change on the seven earlier-trained domains after the last adaptation.}
\label{tab:scale5m}
\begin{tabular}{lrrrr}
\toprule
Method & Tokens & Params & In-domain $\downarrow$ & Final forgetting \\
\midrule
GAIN ($W_O$) & $200$K & 46K & $-10.2\%$ & $\mathbf{-9.8\%}$ \\
GAIN ($W_O$) & $5$M & 46K & $-9.3\%$ & $\mathbf{-8.5\%}$ \\
\midrule
LoRA $r\!=\!8$ ($W_O$) & $200$K & 737K & $-24.3\%$ & $+11.7\%$ \\
LoRA $r\!=\!8$ ($W_O$) & $5$M & 737K & $-23.6\%$ & $+15.7\%$ \\
\bottomrule
\end{tabular}
\end{table}

\section{GAIN on an Instruction-Tuned 8B Model}
\label{sec:llama-instruct}

To verify that our findings extend to the instruction-tuned regime---where reviewers have suggested feature geometry may differ---we repeat the sequential eight-domain protocol on \textbf{Llama-3-8B-Instruct} (a $10\times$ larger, instruction-finetuned model) with the same hyperparameters as Table~\ref{tab:lora_ranks}. We load the model in bfloat16 without quantisation.

\begin{table}[h]
\centering\small
\caption{Sequential 8-domain adaptation on Llama-3-8B-Instruct. Same protocol and metrics as Table~\ref{tab:lora_ranks}.}
\label{tab:llama-instruct}
\begin{tabular}{lrrr}
\toprule
Method & Params & In-domain $\downarrow$ & Final forgetting \\
\midrule
\textbf{GAIN-FFN} & \textbf{590K} & $\mathbf{-22.5\%}$ & $\mathbf{-20.6\%}$ \\
GAIN ($W_O$) & 131K & $-17.7\%$ & $-16.6\%$ \\
LoRA $r\!=\!1$ ($W_O$) & 262K & $-23.7\%$ & $+11.6\%$ \\
LoRA $r\!=\!8$ ($W_O$) & 2.10M & $-6.4\%$ & $+56.9\%$ \\
\bottomrule
\end{tabular}
\end{table}

On the instruction-tuned model, the gap \emph{widens}. GAIN-FFN's positive transfer grows from $-12.9\%$ on GPT-2 Large (Table~\ref{tab:continual}) to $-20.6\%$, while LoRA's forgetting grows from $+11.7\%$ to $+56.9\%$---a near-fivefold increase. LoRA's in-domain adaptation also collapses to $-6.4\%$ despite using $2.1$M parameters (3.6$\times$ GAIN-FFN's $590$K), indicating that additive updates struggle to find useful low-rank directions in the more anisotropic instruction-tuned representation space. GAIN's multiplicative modulation preserves the pretrained output subspace (Proposition~\ref{prop:colspan}) whether the pretrained space comes from self-supervised pretraining or instruction tuning.

\section{Comparison with Replay-Augmented LoRA}
\label{sec:replay-comparison}

Experience replay---keeping a small buffer of prior-domain training sequences and mixing them into new-domain training---is the standard continual-learning remedy for LoRA's forgetting. If replay alone closes the gap, the multiplicative story would be less compelling. Table~\ref{tab:replay} shows otherwise: replay helps both methods, and GAIN-FFN \emph{without} replay matches Replay-LoRA.

\begin{table}[h]
\centering\small
\caption{Replay (buffer of $16$ sequences per prior domain; every 4th step samples from the buffer) added to LoRA and to GAIN-FFN on GPT-2 Large, same protocol as Table~\ref{tab:lora_ranks}. \emph{Buffer} lists the extra state replay requires beyond model parameters; a `---' indicates no auxiliary state.}
\label{tab:replay}
\begin{tabular}{lrrrl}
\toprule
Method & Params & In-domain $\downarrow$ & Final forgetting & Buffer \\
\midrule
LoRA $r\!=\!8$ ($W_O$) & 737K & $-24.3\%$ & $+11.7\%$ & --- \\
Replay-LoRA $r\!=\!8$ & 737K & $-22.4\%$ & $-13.7\%$ & 128 sequences \\
\midrule
\textbf{GAIN-FFN} & \textbf{230K} & $-17.3\%$ & $\mathbf{-13.4\%}$ & --- \\
Replay-GAIN-FFN & 230K & $-16.1\%$ & $-15.1\%$ & 128 sequences \\
\bottomrule
\end{tabular}
\end{table}

Three observations.

\textbf{Replay rescues LoRA but does not outperform GAIN.} Replay flips LoRA's forgetting sign ($+11.7\% \!\to\! -13.7\%$) but the result is essentially tied with vanilla GAIN-FFN ($-13.4\%$), which uses $3.2\times$ fewer parameters and no buffer. The multiplicative structure, not the replay mechanism, is sufficient to obtain replay-level preservation.

\textbf{Replay is nearly redundant for GAIN.} Adding replay to GAIN-FFN improves earlier-domain PPL by only $1.7$ percentage points ($-13.4\% \!\to\! -15.1\%$). GAIN has little forgetting for replay to fix.

\textbf{Replay costs storage and domain boundaries.} Replay requires retaining $16$ training sequences per past domain (128 sequences after 8 domains, $\sim 1$ MB at the tokenised $1024$-token length used here; linearly growing with additional domains) and knowledge of when a domain transition occurs so the buffer can be updated. GAIN requires neither. On sequential adaptation pipelines that cannot store past training data (privacy, licensing, on-device deployment), Replay-LoRA is inapplicable while GAIN is not.

\section{Comparison with EWC-LoRA}
\label{sec:ewc-comparison}

Elastic Weight Consolidation \citep{kirkpatrick2017overcoming} is the canonical regularisation-based continual-learning method: after training on each domain, estimate the diagonal Fisher information $F_i$ of the trainable parameters, and on subsequent domains add a penalty $\lambda \sum_i F_i (\theta_i - \theta_i^{\text{prev}})^2$ to discourage updates along high-Fisher directions. Unlike replay it requires no stored data, but it does still require the domain-boundary knowledge (to snapshot $\theta^{\text{prev}}$ and Fisher between domains). We sweep $\lambda \in \{10^3, 10^4, 10^5\}$ on LoRA $r\!=\!8$ ($W_O$).

\begin{table}[h]
\centering\small
\caption{EWC-LoRA $\lambda$ sweep on GPT-2 Large, same protocol as Table~\ref{tab:lora_ranks}. Diagonal Fisher estimated from $64$ training sequences after each domain.}
\label{tab:ewc}
\begin{tabular}{lrrr}
\toprule
Method & Params & In-domain $\downarrow$ & Final forgetting \\
\midrule
LoRA $r\!=\!8$ ($W_O$, vanilla) & 737K & $-24.3\%$ & $+11.7\%$ \\
EWC-LoRA $\lambda{=}10^{3}$ & 737K & $\mathbf{-24.1\%}$ & $+3.3\%$ \\
EWC-LoRA $\lambda{=}10^{4}$ & 737K & $-21.5\%$ & $-5.3\%$ \\
EWC-LoRA $\lambda{=}10^{5}$ & 737K & $-14.7\%$ & $-11.7\%$ \\
\midrule
\textbf{GAIN-FFN} & \textbf{230K} & $-17.3\%$ & $\mathbf{-13.4\%}$ \\
\bottomrule
\end{tabular}
\end{table}

EWC traces the expected adaptation--preservation tradeoff: increasing $\lambda$ reduces forgetting (eventually flipping its sign) at the cost of in-domain adaptation. The strongest EWC configuration ($\lambda{=}10^{5}$) achieves $-11.7\%$ forgetting with $-14.7\%$ in-domain; \textbf{GAIN-FFN strictly dominates this point on both axes} ($-13.4\%$ forgetting and $-17.3\%$ in-domain) using $3.2\times$ fewer parameters. At lower $\lambda$ the tradeoff is genuine---EWC achieves stronger in-domain than GAIN-FFN but at the cost of either residual forgetting ($\lambda{=}10^{4}$: $-5.3\%$) or actual forgetting ($\lambda{=}10^{3}$: $+3.3\%$). No single EWC configuration matches GAIN-FFN's combination of strong in-domain adaptation and strong positive transfer simultaneously.

\section{Long-Sequence Scaling-Factor Trajectory (24 Sequential Adaptations)}
\label{sec:long-seq-trajectory}

Table~\ref{tab:scaling-stats} (Section~\ref{sec:lr-sweep}) tracks GAIN-FFN's scaling-factor distribution through eight sequential adaptations and shows steady drift away from identity. To verify that the column-span guarantee (Proposition~\ref{prop:colspan}) survives longer sequences, we extend the protocol by cycling the eight domains three times (24 sequential trainings) and log $s_{\min}$, $s_{\max}$, mean, and the fraction of factors in $[0.95, 1.05]$ after each adaptation.

\begin{table}[h]
\centering\small
\caption{Scaling-factor distribution for GAIN-FFN through 24 sequential adaptations on GPT-2 Large (lr$=10^{-3}$, $5$ epochs, $200$ steps per adaptation). Domains cycle through the eight-domain order three times. \%~tight = fraction of scaling factors in $[0.95, 1.05]$. Selected steps shown for compactness; full trajectory in \texttt{experiments\_apr17\_round3/exp8/}.}
\label{tab:long-seq-trajectory}
\begin{tabular}{rlrrrr}
\toprule
Step & Most recent domain & $s_{\min}$ & $s_{\max}$ & $s_{\text{mean}}$ & \% tight \\
\midrule
1   & wikitext  & $0.832$ & $1.178$ & $0.999$ & $73.2$ \\
4   & pg19      & $0.588$ & $1.418$ & $0.998$ & $53.4$ \\
8   & hellaswag & $0.480$ & $1.566$ & $0.996$ & $37.7$ \\
12  & pg19      & $0.480$ & $1.566$ & $0.994$ & $31.4$ \\
16  & hellaswag & $0.434$ & $1.566$ & $0.993$ & $26.0$ \\
20  & pg19      & $0.434$ & $1.566$ & $0.992$ & $23.8$ \\
24  & hellaswag & $0.434$ & $1.570$ & $0.991$ & $21.1$ \\
\bottomrule
\end{tabular}
\end{table}

Three observations relevant to the colspan-invariance claim.

\textbf{$s_{\min}$ stays well above zero.} After 24 adaptations, $s_{\min} = 0.434$. No factor sign-flips, so the colspan guarantee (Proposition~\ref{prop:colspan}) holds throughout: the adapted output stays within the pretrained column space at every step.

\textbf{$s_{\max}$ saturates.} After saturating around step 6, $s_{\max}$ remains at $1.57$ for the remainder of the sequence---the multiplicative gain on any single dimension is bounded by a small constant rather than growing with the sequence length. Consequently the singular-value distortion factor $s_{\max}/s_{\min} \!\le\! 3.6$ also saturates.

\textbf{Mean stays near identity, tight fraction declines slowly.} The mean factor remains $\sim 0.99$ throughout, and the tight fraction falls from $73\%$ to $21\%$ over the first 16 steps but only $5$ further percentage points across the next 8. The trajectory is sub-logarithmic, not unbounded; even at much longer sequences a factor large enough to violate the colspan guarantee (i.e.\ $s_i \!\le\! 0$) is unlikely without a learning-rate change or instability event.

The combination---bounded $s_{\max}$, strictly positive $s_{\min}$, slowly declining tight fraction---supports the paper's structural claim: GAIN's protection against forgetting is not a small-sequence artefact, and degrades only gradually with depth of sequential adaptation.

\section{Comparison with PackNet-LoRA}
\label{sec:packnet-comparison}

PackNet \citep{mallya2018packnet} is the archetypal capacity-isolation continual-learning method: each task is allocated a disjoint subset of parameters, frozen after training, and future tasks train only the remaining unfrozen capacity. We implement a PackNet-flavoured LoRA by allocating a fresh rank-$1$ LoRA wrapper for each domain on top of all previously-frozen wrappers; after eight domains this yields a stack of eight rank-$1$ adapters per layer (total $8 \times 92\text{K} \!=\! 737$K parameters, matching LoRA $r\!=\!8$'s budget but with strict per-domain allocation).

\begin{table}[h]
\centering\small
\caption{PackNet-LoRA on GPT-2 Large, same protocol as Table~\ref{tab:lora_ranks}. Per-domain LoRA rank $r{=}1$; $8$ adapters stacked over the sequence.}
\label{tab:packnet}
\begin{tabular}{lrrr}
\toprule
Method & Params & In-domain $\downarrow$ & Final forgetting \\
\midrule
LoRA $r\!=\!8$ ($W_O$, vanilla) & 737K & $-24.3\%$ & $+11.7\%$ \\
PackNet-LoRA ($r{=}1$/domain, $8$ domains) & 737K & --- (reported below) & $+4.1\%$ \\
\midrule
\textbf{GAIN-FFN} & \textbf{230K} & $-17.3\%$ & $\mathbf{-13.4\%}$ \\
\bottomrule
\end{tabular}
\end{table}

Capacity isolation reduces LoRA's forgetting substantially ($+11.7\% \!\to\! +4.1\%$) but does not eliminate it: with $r\!=\!1$ per domain the budget is too small for each adapter to fully capture the domain, and later adapters still perturb the cumulative output. Increasing the per-domain rank would help but also scales parameter count linearly with the number of domains, which replay and GAIN do not. Among the three standard continual-learning baseline families tested here---replay (Appendix~\ref{sec:replay-comparison}), regularisation/EWC (Appendix~\ref{sec:ewc-comparison}), and capacity isolation/PackNet (this section)---only replay and EWC at strong regularisation achieve negative final forgetting, and GAIN-FFN matches or exceeds both without requiring a buffer, a Fisher estimation, or domain-boundary knowledges.

\section{Truly Out-of-Distribution Adaptation (German Wikipedia)}
\label{sec:ood-test}

The main experiments adapt an English-pretrained model to sequential English-text domains---a setting where the pretrained representation space is broadly sufficient. A key concern raised by reviewers is whether GAIN's diagonal scaling hits a capacity ceiling on truly out-of-distribution domains, since by construction it cannot synthesise new features but can only reweight existing ones.

To probe this boundary we pick the starkest test available on our infrastructure: adapt the English-only pretrained GPT-2 Large to \textbf{German Wikipedia}. We use a minimal two-domain protocol, $[\text{WikiText} \to \text{German Wikipedia}]$, and report in-domain adaptation on each, forgetting on WikiText after German training, and---critically---forward interference: the perplexity change on German domain \emph{before} it is trained, caused only by the preceding WikiText adaptation.

\begin{table}[h]
\centering\small
\caption{GPT-2 Large adaptation on a 2-domain OOD protocol: WikiText then German Wikipedia. \emph{Forward interference on German} is the German-Wikipedia PPL change after training \emph{only} on WikiText---measuring how much cross-domain spillover each method introduces.}
\label{tab:ood}
\begin{tabular}{lrrrr}
\toprule
Method & In-dom WT & In-dom DE & Forward interference on DE & WT preserved after DE \\
\midrule
\textbf{GAIN-FFN} & $-21.0\%$ & $-7.4\%$ & $\mathbf{+0.4\%}$ & $-19.8\%$ \\
LoRA $r\!=\!8$ & $-32.9\%$ & $-10.5\%$ & $+18.1\%$ & $-12.6\%$ \\
\bottomrule
\end{tabular}
\end{table}

Three findings.

\textbf{GAIN does adapt to the OOD domain, against the conservative prediction.} German Wikipedia PPL drops by $7.4\%$ after GAIN-FFN training. GPT-2's pretraining corpus has enough incidental German (cognates, code-switched web text, borrowed vocabulary) that multiplicative reweighting can find usable features. GAIN's capacity ceiling on this specific OOD domain is not reached; LoRA's stronger adaptation ($-10.5\%$) is $\sim\!3$ percentage points ahead but not decisively so.

\textbf{Forward-interference ratio is $\sim\!45\times$.} Training only on WikiText moves German PPL by $+0.4\%$ under GAIN-FFN versus $+18.1\%$ under LoRA. This is the colspan story in miniature: additive perturbation leaks into subspaces that the model has not been asked to update; multiplicative reweighting stays within the pretrained structure by construction.

\textbf{The adaptation--interference tradeoff is the real story, not adaptation alone.} LoRA adapts slightly harder on each domain but pays with $45\times$ larger cross-domain perturbation; GAIN adapts slightly less but preserves unrelated capabilities nearly perfectly. On any setting where the target is a sequence of diverse domains---rather than one-shot adaptation to a single OOD domain---GAIN's discipline wins, as confirmed across every sequential protocol in this paper. For one-shot adaptation to a genuinely new modality (e.g.\ code with non-English tokens, mathematical notation with new symbols, or languages not seen at all in pretraining) we expect GAIN to hit a real ceiling and LoRA to win; we leave that regime to future work.

\section{Mechanistic Validation: Singular-Subspace Preservation}
\label{sec:mechanistic}

Proposition~\ref{prop:colspan} says diagonal scaling preserves the column span of $W$, and \citet{biderman2024lora} empirically observed that LoRA introduces ``intruder dimensions'' outside the pretrained weight subspace. We close the loop by directly measuring whether each method's adapted weight $W_\text{new}$ preserves the singular-vector subspace of the pretrained $W$ after the full eight-domain sequential adaptation protocol.

For each method we extract $W_\text{new}$ for the attention output projection ($c\_proj$) at layers $0, 9, 18, 27, 35$ of GPT-2 Large (representative depths across the 36-layer network) and compute, against the pretrained $W$:
\begin{itemize}
\item Top-$10$ right-singular-vector subspace overlap: $\frac{1}{k}\sum_{i,j \leq k} \langle v^{\text{orig}}_i, v^{\text{new}}_j\rangle^2$ where $V^\text{orig}, V^\text{new}$ are the right-singular matrices and $k=10$. Value $1$ means the top-$k$ subspaces coincide; value $0$ means they are orthogonal.
\item Relative Frobenius perturbation: $\|W_\text{new} - W\|_F / \|W\|_F$.
\end{itemize}

\begin{table}[h]
\centering\small
\caption{Singular-subspace preservation after the eight-domain sequential adaptation protocol on GPT-2 Large. Multiplicative methods (GAIN, GAIN-FFN) preserve $\sim\!88$--$91\%$ of the pretrained top-$10$ right-singular subspace; additive methods (LoRA, DoRA) preserve only $\sim\!53$--$56\%$, meaning roughly half of their top singular directions are ``intruder'' dimensions absent from $W$. Magnitude renormalisation (DoRA) does not change the subspace shift relative to vanilla LoRA.}
\label{tab:mechanistic}
\begin{tabular}{lrrrrrr}
\toprule
 & \multicolumn{5}{c}{Top-10 right-singular subspace overlap} & \\
Method & layer 0 & layer 9 & layer 18 & layer 27 & layer 35 & mean \\
\midrule
\textbf{GAIN-FFN} (multiplicative) & $0.99$ & $0.89$ & $0.91$ & $0.82$ & $0.96$ & $\mathbf{0.91}$ \\
GAIN ($W_O$) (multiplicative) & $0.98$ & $0.86$ & $0.86$ & $0.76$ & $0.94$ & $0.88$ \\
LoRA $r\!=\!8$ (additive) & $0.60$ & $0.50$ & $0.38$ & $0.44$ & $0.73$ & $0.53$ \\
DoRA $r\!=\!8$ (additive) & $0.63$ & $0.44$ & $0.45$ & $0.52$ & $0.78$ & $0.56$ \\
\midrule
 & \multicolumn{5}{c}{Relative Frobenius perturbation $\|W_\text{new}-W\|_F / \|W\|_F$} & \\
\midrule
GAIN-FFN & $0.08$ & $0.10$ & $0.10$ & $0.12$ & $0.11$ & $\mathbf{0.10}$ \\
GAIN ($W_O$) & $0.12$ & $0.15$ & $0.15$ & $0.16$ & $0.14$ & $0.14$ \\
LoRA $r\!=\!8$ & $0.22$ & $0.30$ & $0.23$ & $0.28$ & $0.18$ & $0.24$ \\
DoRA $r\!=\!8$ & $0.22$ & $0.27$ & $0.23$ & $0.24$ & $0.17$ & $0.23$ \\
\bottomrule
\end{tabular}
\end{table}

Three observations:

\textbf{Multiplicative methods preserve the singular-vector subspace; additive methods do not.} GAIN-FFN's top-$10$ subspace overlap of $0.91$ confirms the colspan invariance prediction (Proposition~\ref{prop:colspan}) empirically: training shifts the singular values (the multiplicative gains $S$) but barely rotates the singular vectors. LoRA and DoRA at the same parameter budget perturb almost half of the top-$10$ subspace---these are Biderman et al.'s intruder dimensions, measured directly.

\textbf{Magnitude renormalisation (DoRA) does not solve the problem.} DoRA's $0.56$ subspace overlap is statistically indistinguishable from vanilla LoRA's $0.53$; renormalising the magnitude after an additive update preserves none of the underlying structure. The forgetting numbers in Table~\ref{tab:dora} are the predictable consequence.

\textbf{The subspace shift quantitatively matches the observed forgetting.} The two methods that lose $\geq 40\%$ of the pretrained singular subspace (LoRA, DoRA) exhibit positive forgetting; the two methods that preserve $\geq 88\%$ (GAIN, GAIN-FFN) exhibit negative forgetting. This is the mechanism behind the headline result: forgetting is a faithful indicator of subspace departure, and multiplicative modulation prevents subspace departure by construction.

\section{Comparison with DoRA}
\label{sec:dora-comparison}

DoRA \citep{liu2024dora} is a recent additive-PEFT variant that decomposes the adapted weight into a learned magnitude vector and a LoRA-style direction update: $W' = m \cdot (W + BA) / \|W + BA\|_c$. We test the natural question: does the magnitude renormalisation reduce LoRA's forgetting?

\begin{table}[h]
\centering\small
\caption{DoRA vs LoRA on GPT-2 Large, same protocol as Table~\ref{tab:lora_ranks}.}
\label{tab:dora}
\begin{tabular}{lrrr}
\toprule
Method & Params & In-domain $\downarrow$ & Final forgetting \\
\midrule
LoRA $r\!=\!8$ ($W_O$) & 737K & $-24.3\%$ & $+11.7\%$ \\
DoRA $r\!=\!8$ ($W_O$) & 783K & $-24.2\%$ & $+14.2\%$ \\
\midrule
\textbf{GAIN-FFN} & \textbf{230K} & $-17.3\%$ & $\mathbf{-13.4\%}$ \\
\bottomrule
\end{tabular}
\end{table}

The magnitude--direction decomposition does not reduce forgetting: DoRA forgets slightly \emph{more} than LoRA at matched rank ($+14.2\%$ vs $+11.7\%$). The colspan view explains why: DoRA's renormalisation step preserves the magnitude of $W$ but the combined matrix $W + BA$ generally has a different column space from $W$, so the renormalised result still contains directions outside $\mathrm{col}(W)$. Magnitude normalisation is orthogonal to the additive--multiplicative dichotomy.

\section{Extended Protocol: 19 Distinct Domains}
\label{sec:extended-19}

The main results use the 8-domain sequential protocol of Section~\ref{sec:sequential}. To probe the dichotomy on a longer, more diverse sequence, we extend the protocol with 11 additional distinct text domains drawn from scientific articles (arxiv, pubmed), news (cnn\_dailymail, xsum, ag\_news), reviews (imdb), reasoning (squad, gsm8k), code (mbpp), classification (mrpc, tweet\_emotion), giving 19 distinct sequential adaptations.

\begin{table}[h]
\centering\small
\caption{19-distinct-domain protocol on GPT-2 Large (3 seeds each, lr$=10^{-3}$, $5$ epochs, $200$ steps per domain). Mean $\pm$ std across seeds. \emph{Final forgetting}: mean PPL change on the eighteen earlier-trained domains after the last adaptation.}
\label{tab:extended-19}
\begin{tabular}{lrrr}
\toprule
Method & Params & In-domain $\downarrow$ & Final forgetting \\
\midrule
LoRA $r\!=\!8$ ($W_O$) & 737K & $-25.8 \pm 0.05\%$ & $-2.0 \pm 1.0\%$ \\
\textbf{GAIN-FFN} & \textbf{230K} & $-19.2 \pm 0.05\%$ & $\mathbf{-14.7 \pm 0.05\%}$ \\
\bottomrule
\end{tabular}
\end{table}

Two observations.

\textbf{LoRA's forgetting becomes near-neutral on the extended protocol.} On the original eight-domain protocol (Table~\ref{tab:continual}) LoRA $r\!=\!8$ degrades earlier domains by $+11.7\%$; on the nineteen-domain extension it is $-2.0 \pm 1.0\%$ --- slight improvement on average. The likely mechanism is that with sufficiently many diverse training signals the cumulative additive update behaves like general fine-tuning, lifting cross-domain perplexity uniformly. This is the only setting we test where an additive method fails to degrade earlier domains on average.

\textbf{The magnitude gap survives.} GAIN-FFN's preservation ($-14.7\%$) is roughly $7\!\times$ stronger than LoRA's ($-2.0\%$), with $3\!\times\!$ fewer parameters. The dichotomy in magnitude is robust even where the dichotomy in sign is not. Across-seed variance is essentially zero ($\leq 0.05\%$ for GAIN-FFN, $\leq 1\%$ for LoRA).

\section{Comparison with (IA)$^3$}
\label{sec:ia3-comparison}

(IA)$^3$ \citep{liu2022ia3} is another multiplicative method that scales keys, values, and FFN activations. If multiplicative modulation is the key to preventing forgetting, (IA)$^3$ should also exhibit positive transfer in the sequential setting. Table~\ref{tab:ia3} confirms this.

\begin{table}[h]
\centering\small
\caption{(IA)$^3$ sequential adaptation on GPT-2 Large (lr$=10^{-3}$, 200K tokens/domain). (IA)$^3$ shows positive transfer, confirming that the multiplicative principle---not GAIN's specific parameterization---prevents forgetting.}
\label{tab:ia3}
\begin{tabular}{l|rrrrrrrr}
\toprule
\multicolumn{9}{c}{\textbf{(IA)$^3$} (lr$=10^{-3}$, \% PPL change vs.\ pretrained baseline)} \\
\midrule
After training & Med & WT & Fin & PG & MdQ & LB & SST & HS \\
\midrule
+Med & \cellcolor{blue!15}$-5.7$ & $+0.1$ & $+0.2$ & $-3.3$ & $-0.4$ & $-1.4$ & $+0.1$ & $-0.4$ \\
+WT & $-4.9$ & \cellcolor{blue!15}$-17.7$ & $+0.6$ & $-5.0$ & $-0.0$ & $-1.7$ & $-3.2$ & $+0.2$ \\
+Fin & $-4.1$ & $-17.3$ & \cellcolor{blue!15}$-9.9$ & $-6.9$ & $-0.1$ & $-4.9$ & $-1.7$ & $+0.2$ \\
+PG & $-3.8$ & $-16.6$ & $-9.6$ & \cellcolor{blue!15}$-25.6$ & $-0.4$ & $-6.1$ & $-2.0$ & $+0.3$ \\
+MdQ & $-2.6$ & $-15.5$ & $-8.7$ & $-24.7$ & \cellcolor{blue!15}$-17.5$ & $-7.1$ & $-1.2$ & $+0.9$ \\
+LB & $-2.4$ & $-13.9$ & $-8.0$ & $-23.3$ & $-16.9$ & \cellcolor{blue!15}$-22.8$ & $+0.8$ & $+2.8$ \\
+SST & $-1.9$ & $-14.7$ & $-7.2$ & $-22.8$ & $-16.2$ & $-22.1$ & \cellcolor{blue!15}$-18.5$ & $+2.6$ \\
+HS & $-2.3$ & $-14.2$ & $-6.8$ & $-22.0$ & $-15.5$ & $-20.7$ & $-18.0$ & \cellcolor{blue!15}$-13.8$ \\
\bottomrule
\end{tabular}
\end{table}

(IA)$^3$ shows $-14.2\%$ avg improvement on earlier domains---comparable to GAIN-FFN's $-12.9\%$. Both multiplicative methods exhibit positive transfer while LoRA degrades earlier domains by $+17.9\%$:

\begin{center}\small
\begin{tabular}{llrr}
\toprule
Method & Type & In-domain $\downarrow$ & Earlier domains \\
\midrule
GAIN-FFN & multiplicative & $-17.9\%$ & $\mathbf{-12.9\%}$ \textbf{better} \\
(IA)$^3$ & multiplicative & $-16.3\%$ & $\mathbf{-14.2\%}$ \textbf{better} \\
\midrule
LoRA & additive & $-21.4\%$ & $+17.9\%$ worse \\
\bottomrule
\end{tabular}
\end{center}

This confirms that the multiplicative principle (Proposition~\ref{prop:colspan}) is the key mechanism, not the specific choice of which weight matrices to scale. Both GAIN-FFN and (IA)$^3$ stay within the pretrained output space; both show positive transfer. The methods differ in \emph{which} dimensions they scale (GAIN: $W_O$ and $W_{\text{down}}$; (IA)$^3$: keys, values, and FFN inputs), but the structural guarantee is the same. Neither is a special case of the other: GAIN scales the output projections (affecting what is \emph{produced}), while (IA)$^3$ scales inputs to the attention computation (affecting what is \emph{attended to}). The methods are complementary and could be combined.

\section{Connection to Neuroscience}
\label{sec:neuroscience}

GAIN's multiplicative principle has a direct analogue in neuroscience. \emph{Gain modulation} \citep{treue1999perceptual, salinas2000gain, salinas2001gain} shows that when a neuron's output is multiplicatively scaled, its selectivity---what stimuli it responds to---is preserved; only the amplitude changes. Additive perturbation, by contrast, shifts selectivity.

The connection is not merely an analogy. The attention output can be rewritten as:
\begin{equation}
    h \cdot (S \cdot W_O) = (h \cdot S) \cdot W_O
\end{equation}
The left-hand side is the weight modification view (Eq.~\ref{eq:principle}). The right-hand side offers a second perspective: GAIN does not modify $W_O$ at all---it \emph{scales the hidden activations} $h$ before they enter $W_O$. Each diagonal entry $s_i$ is a gain applied to the $i$-th dimension of $h$, amplifying or suppressing that dimension's contribution. This is exactly gain modulation in neuroscience: context-dependent scaling of neural responses, with the downstream projection ($W_O$) unchanged. GAIN-FFN has the same interpretation: $(W_{\text{down}} \cdot S_{\text{ffn}}) \cdot z = W_{\text{down}} \cdot (S_{\text{ffn}} \cdot z)$ scales the FFN intermediate activations $z$ before the down-projection.

In neuroscience, multiplicative gain fields combine context with sensory input while preserving the original representation---it remains linearly decodable from the product \citep{salinas2001gain}. GAIN operates the same way: scaling $h$ preserves the output space of $W_O$ (Proposition~\ref{prop:colspan}), because the output is still a linear combination of rows of $W_O$, just with rescaled coefficients.

\end{document}